\documentclass{article}
\usepackage{theapa, rawfonts}

\usepackage{algorithm} 
\usepackage{algpseudocode} 
\usepackage{mathrsfs} 
\usepackage{booktabs} 
\usepackage{tabularx} 
\usepackage{float} 
\usepackage{multirow} 
\usepackage{mymath} 
\usepackage{tikz-setup} 
\usepackage{natbib} 
\usepackage{cancel} 
\usepackage{balance}
\usepackage{hyperref}
\usepackage{lineno}
\usepackage{subcaption} 

\newcommand{\vasilii}[1]{\textcolor{blue}{#1}}

\renewcommand{\emph}[1]{\textit{#1}}

\begin{document}

\title{Multi-class Probabilistic Bounds for Self-learning}

\author{Vasilii Feofanov, Emilie Devijver, Massih-Reza Amini\\
\{Firstname.LastName\}@univ-grenoble-alpes.fr  \\
Univ. Grenoble Alpes, CNRS, Grenoble INP, LIG\\
Grenoble, France}

\maketitle

\begin{abstract}
Self-learning is a classical approach for learning with both labeled and unlabeled observations which consists in giving pseudo-labels to unlabeled training instances with a confidence score over a predetermined threshold. At the same time, the pseudo-labeling technique is prone to error and runs the risk of adding noisy labels into unlabeled training data. In this paper, we present a probabilistic framework for analyzing self-learning in the multi-class classification scenario with partially labeled data. First, we derive a transductive bound over the risk of the multi-class majority vote classifier. Based on this result, we propose to automatically choose the threshold for pseudo-labeling that minimizes the transductive bound.
Then, we introduce a mislabeling error model to analyze the error of the majority vote classifier in the case of the pseudo-labeled data. We derive a probabilistic C-bound over the majority vote error when an imperfect label is given. Empirical results on different data sets show the effectiveness of our framework compared to several state-of-the-art semi-supervised approaches.
\end{abstract}

\section{Introduction}

We consider classification problems where the scarce labeled training set comes along with a huge number of unlabeled training examples.
This is for example the case in web-oriented applications where a huge number of unlabeled observations arrive sequentially, and there is not enough time to manually label them all.

In this context, the use of traditional supervised approaches trained on available labeled data usually leads to poor learning performance.
In semi-supervised learning \citep{Chapelle:2010}, it is generally assumed that unlabeled training examples contain valuable information about the prediction problem, so the aim  is to exploit \textit{both} available labeled and unlabeled training observations in order to provide an improved solution.
The self-learning\footnote{It is also known as self-training or self-labeling.} \citep{Tur:2005,Amini:15} is a classical approach to classify partially labeled data in a supervised fashion, where the training set is augmented by iteratively assigning pseudo-labels to unlabeled examples with the confidence score above a certain threshold. However, fixing this threshold is a bottleneck of this approach. In reality, at every iteration, the self-learning algorithm injects some noise in labeling, so the question would be how to optimally choose the threshold to minimize the mislabeling probability.

In this paper, we tackle this problem from a theoretical point of view for the multi-class classification case and analyze the behavior of majority vote classifiers (also known as Bayes classifiers, including Random Forest \citep{Lorenzen:2019}, AdaBoost \citep{Germain:2015}, SVM \citep{Fakeri-Tabrizi:2015} and neural networks \citep{Letarte:2019}) for semi-supervised learning. The majority vote classifier is well studied in the binary case, where a classical approach is to bound the majority vote risk indirectly by twice the risk of related stochastic Gibbs classifier \citep{Langford:2003,Begin:2014}. However, the voters may compensate the errors of each other, so the majority vote risk will be much smaller than the Gibbs risk.

In the transductive setting \citep[p. 339]{Vapnik:1998}, where the aim is to correctly classify unlabeled training examples, \citet{Feofanov:2019} derived a bound for the multi-class majority vote classifier by analyzing distribution of the class vote, focusing on the class confusion matrix as an error indicator as proposed by \citet{Morvant:2012:ICML}. This bound is obtained by analytically solving a linear program and it comes out that in the case when the majority vote classifier makes most of its errors on examples with low class vote, the obtained bound is tight. This result is proposed to develop a new multi-class self-learning algorithm where the threshold is automatically found based on the proposed transductive bound.
%
%
Our paper extends this work by deriving the transductive bounds in the probabilistic framework. In this case, the transductive bound is estimated by assigning soft labels for the unlabeled set, which is more effective in practice as pointed out by \cite{Feofanov:2019}, so it bridges the gap between the theoretical analyzes and the application.
Subsequently, we theoretically analyze the behavior of the majority vote classifier after the inclusion of pseudo-labeled training examples by self-learning. Even when the threshold is optimally chosen, the pseudo-labels may still be erroneous, so the question is how to evaluate the risk in this noisy case. For this,
we take explicitly into account possible mislabeling by considering a mislabeling model of \cite{Chittineni:1980}. At first, we show the connection between the classification error of the true and the imperfect label. Then, we derive a new probabilistic C-bound over the error of the multi-class majority vote classifier in the presence of imperfect labels. This bound is based on the mean and the variance of the prediction margin \citep{Lacasse:2007}, so it reflects both the individual strength of voters and their correlation in prediction.

The rest of this paper is organized as follows. Section \ref{sec:rel-work} provides an overview of the related work. In Section \ref{sec:framework} we introduce the problem statement and the proposed framework. In Section \ref{sec:tr-study} we present a probabilistic bound over the transductive risk of the multi-class majority vote classifier and describe the extended self-learning algorithm that learns the threshold using the proposed bound. Section \ref{sec:c-bound} shows how to derive the C-bound in the probabilistic framework taking into account mislabeling errors. In Section \ref{sec:num-exper}, we present empirical evidence showing that the proposed self-learning strategy is effective compared to several state-of-the-art approaches, and we illustrate the behavior of the new C-bound on real data sets. Finally, in Section \ref{sec:concl} we summarize the outcome of this study and discuss the future work.

\section{\vasilii{Related Work}}
\label{sec:rel-work}
Generalization guarantees of majority vote classifiers are well studied in the binary supervised setting.
A common approach is to bound the majority vote risk by twice the Gibbs risk \citep{Langford:2003}. Many works are focused on deriving tight PAC guarantees for the Gibbs classifier in the inductive case \citep{McAllester:2003,Maurer:2004,Catoni:2007} as well as in the transductive one \citep{Derbeko:2004,Begin:2014}, and applying these results for optimization \citep{Thiemann:2017}, linear classifiers \citep{Germain:2009}, random forests \citep{Lorenzen:2019}, neural networks \citep{Letarte:2019}. While this bound can be tight, it reflects only the individual strength of voters, so using it as a minimization criterion often leads to an increase in the test error
\citep{Masegosa:2020}. This motivates to opt for bounds that directly upper bound the majority vote error. \cite{Amini:2008} derives a transductive bound based on how voters agree on every unlabeled example, while \cite{Lacasse:2007} upper bounds the generalization error by taking additionally into account the error correlation between voters.

Only few results exist for the multi-class majority vote classifier. In the supervised setting, \citet{Morvant:2012:ICML} derives generalization guarantees on the confusion matrix' norm, whereas \citet{Laviolette:2017} extends the C-bound of \cite{Lacasse:2007} to the multi-class case. \cite{Masegosa:2020} studies tight estimations from data by deriving a relaxed version of \cite{Laviolette:2017}. In the transductive setting, \citet{Feofanov:2019} extends the bound of \cite{Amini:2008} to the multi-class case. In this paper, we show how the bounds of \citet{Feofanov:2019} and \citet{Laviolette:2017} are generalized to the probabilistic framework.

However, the aforementioned studies are limited by assuming that all training examples are
perfectly labeled. Learning with an imperfect supervisor, in which training data contains an unknown portion of imperfect labels, has been considered in both supervised \citep{Natarajan:2013,Scott:2015,Xia:2019} and semi-supervised settings \citep{Amini:2003}. 
In most cases, a focus is on the estimation of the mislabeling errors to train a classifier, and 
theoretical studies are limited by the binary case \citep{Natarajan:2013,Scott:2015}.  \cite{Chittineni:1980} analyzes the connection between the true and the imperfect label in the multi-class case but only for the maximum a posteriori classifier. We extend the latter result to an arbitrary classifier and use it to derive a new C-bound with imperfect labels. To the best of our knowledge, the majority vote classifier has not been yet studied in the presence of imperfect labels.

In this paper, our theoretical development has a particular focus on semi-supervised learning. While there exists theoretical analysis of graph-based \citep{El:2009} and clustering \citep{Rigollet:2007,Maximov:2018} approaches, little attention is given to a self-learning algorithm \citep{Tur:2005}. A common approach is to perform self-learning with a fixed threshold; another method is to control the number of pseudo-labeled examples by curriculum learning \citep{Cascantebonilla:2020}. We show that this threshold can be effectively found at every iteration as a trade-off between the number of pseudo-labeled examples and the bounded transductive error evaluated on them.  


\section{Framework and Definitions}
\label{sec:framework}
We consider multi-class classification problems with an input space $\mathcal{X}\subset \R^d$ and an output space $\mathcal{Y}=\{1,\dots,K\}$, $K\geq 2$. We denote by $\mbf{X}=(X_1,\ldots,X_d)\in\mathcal{X}$ (resp. $Y\in\mathcal{Y}$)  an input (resp. output) random variable. Considering the semi-supervised framework, we assume an available set of labeled training examples $\mathrm{Z}_{\mathcal{L}}=\{(\mathbf{x}_i,y_i)\}_{i=1}^l\in(\mathcal{X}\times\mathcal{Y})^l$, identically and independently distributed  (i.i.d.) with respect to a fixed yet unknown probability distribution $P(\mbf{X}, Y)$ over $\mathcal{X}\times\mathcal{Y}$, and an available set of unlabeled training examples $\mathrm{X}_{\sss\mathcal{U}} = \{\mathbf{x}_i\}_{i=l+1}^{l+u}\in \mathcal{X}^u$ is supposed to be drawn i.i.d. from the marginal distribution $P(\mathbf{X})$, over the domain $\mathcal{X}$. Further, we denote by $\mathbf{0}_K$ the zero vector of size $K$, $\mathbf{0}_{K,K}$ is the zero matrix of size $K\times K$ and $n:=l+u$.

In this work, a fixed class of classifiers $\mathcal{H}=\{h | h:\mathcal{X} \rightarrow \mathcal{Y}\}$, called the \emph{hypothesis space}, is considered and defined without reference to the training set.
Over $\mathcal{H}$, two probability distributions are introduced: the prior $P$ and the posterior $Q$ that are defined respectively before and after observing the training set.
We focus on two classifiers: the \emph{$Q$-weighted majority vote classifier} (also called the Bayes classifier)\footnote{For the sake of brevity, we will tend to use the latter name, which should not be confused with other learning paradigms based on the Bayesian inference, e.g., the Bayesian statistics.} defined for all $\mathbf{x}\in \mathcal{X}$ as:
\begin{equation}
\label{eq:Bayes-classifier-multi-again}
B_Q(\mathbf{x}):= \argmax_{c\in\{1,\ldots,K\}}\left[\E_{h\sim Q}\I{h(\mathbf{x})=c}\right],
\end{equation}
and, the stochastic \emph{Gibbs classifier} $G_Q$ that for every $\mathbf{x}\in\mathcal{X}$ predicts the label using a randomly chosen classifier $h\in\mathcal{H}$ according to $Q$. The former one represents a class of learning methods, where the predictions of hypotheses are aggregated using the majority vote rule scheme, while the latter one is often used to analyze the behavior of the Bayes classifier.

The goal of learning is formulated as to choose a posterior distribution $Q$ over $\mathcal{H}$ based on the training set $\mathrm{Z}_{\mathcal{L}}\cup \mathrm{X}_{\sss\mathcal{U}}$ such that the classifier $B_Q$ will have the smallest possible error value. As opposed to works of \citet{Derbeko:2004,Begin:2014,Feofanov:2019}, who considered the deterministic case where for each unlabeled example there is one and only one possible label, in this paper we consider the more general \emph{probabilistic} case assuming possibility of multiple outcomes for each example.

To measure confidence of the majority vote classifier in its prediction, the notions of {class votes} and {margin} are further considered.
Given an observation $\mathbf{x}$, we define a vector of \emph{class votes} $\mathbf{v}_\mathbf{x} = (v_Q(\mathbf{x}, c))^K_{c=1}$ where the $c$-th component corresponds to the total vote given to the class $c$:
\begin{equation*}
\label{eq:class-vote}
v_Q(\mathbf{x}, c) := \E_{h\sim Q}\I{h(\mathbf{x})=c} = \sum_{h: h(\mathbf{x})=c} Q(h).    
\end{equation*}
In practice, the vote $v_Q(\mbf{x}, c)$ can be regarded as an estimation of the posterior probability $P(Y\!=\!c|\mbf{X}\!=\!\mbf{x})$; a large value indicates high confidence of the classifier that the true label of $\mathbf{x}$ is $c$.
\\
Given an observation $\mathbf{x}$, its \emph{margin} is defined in the following way:
\begin{align}
M_Q(\mathbf{x}, y) &:= \E_{h\sim Q}\I{h(\mathbf{x})=y} - \max_{\substack{{c\in\mathcal{Y}}\\{c\neq y}}} \E_{h\sim Q}\I{h(\mathbf{x})=c} = v_Q(\mbf{x},y) - \max_{\substack{{c\in\mathcal{Y}}\\{c\neq y}}} v_Q(\mbf{x},c).\label{def:margin}
\end{align}
The margin measures a gap between the vote of the true class and the maximal vote among all other classes. If the value is strictly positive for an example $\mbf{x}$, then $y$ will be the output of the majority vote, so the example will be correctly classified.




\section{Probabilistic Transductive Bounds and Their Application}
\label{sec:tr-study}
In this section, we derive guarantees for the multi-class majority vote classifier in the transductive setting  \citep{Vapnik:1982,Vapnik:1998}, i.e., when the error is evaluated on the unlabeled set $\mathrm{X}_{\sss\mathcal{U}}$ only. The proposed bound assumes that the majority vote classifier makes mistake on low class votes and thereby use votes as indicators of confidence. Then, we propose an application for generic self-learning where the threshold is based on the bound minimization.

\subsection{Transductive conditional risk}
\label{sec:tr-bound-theory}

At first, we show how to upper bound the risk evaluated conditionally to the values of the true and the predicted class.
Given a classifier $h$, for each class pair $(i,j)\in\{1,\dots,K\}^2$ such that $i\not= j,$ the \emph{transductive conditional risk} is defined as follows:
$$
R_\mathcal{U}(h,i,j) := \frac{1}{u_i} \sum_{\mathbf{x}\in \mathrm{X}_{\sss\mathcal{U}}}P(Y=i|X=\mathbf{x}) \I{h(\mathbf{x}) = j},
$$
where $u_i = \sum_{\mathbf{x}\in \mathrm{X}_{\sss\mathcal{U}}} P(Y=i|X=\mathbf{x})$ is the expected number of unlabeled observations from the class $i\in\{1,\ldots,K\}$. The value of $R_\mathcal{U}(h,i,j)$ indicates the expected proportion of unlabeled examples that are classified to the class $j$ being from the class $i$. We call $R_\mathcal{U}(B_Q,i,j)$ as the \emph{transductive Bayes conditional risk}.
In the similar way, the \emph{transductive Gibbs conditional risk} is defined for all $(i,j)\in\{1,\dots,K\}^2,\ i\neq j$ by:
$$
R_\mathcal{U}(G_Q,i,j) := \E_{h\sim Q} R_\mathcal{U}(h,i,j).
$$
Although the Gibbs classifier is stochastic, its error is defined in expectation over $Q$. In other words, the Gibbs conditional risk represents the $Q$-weighted average conditional risk of hypotheses $h\in\mathcal{H}$.

In addition, we define the transductive \textit{joint} Bayes conditional risk for a threshold vector $\bm{\theta}\in[0,1]^K$, for $(i,j)\in\{1,\dots,K\}^2,\ i\neq j$, as follows:
\begin{align}
R_{\mathcal{U}\wedge\bm{\theta}}(B_Q,i,j) &:= \frac{1}{u_i} \sum_{\mathbf{x}\in \mathrm{X}_{\sss\mathcal{U}}} P(Y=i|X=\mathbf{x}) \I{B_Q(\mathbf{x}) = j}\I{v_Q(\mathbf{x},j)\geq \theta_j}. \label{def:joint-cond-risk} \end{align}
If the Bayes classifier makes mistakes, i.e., outputs the class $j$ when the true class is $i$, on the examples with low values of $v_Q(\mbf{x},j)$, then the joint  risk  computes  the  probability  to  make  the conditional error on  confident observations when a large enough $\theta_{j}$ is set with respect to the distribution of $v_Q(\mbf{x},j)$.

The following Lemma \ref{lem:connection-Gibbs-Bayes-multi} connects the conditional Gibbs risk and the joint Bayes conditional risk by considering a conditional Bayes error regarding a certain class vote.

\begin{lem}
\label{lem:connection-Gibbs-Bayes-multi}
For $c\in\{1,\ldots,K\}$, let $\Gamma_{c} = \{\gamma_{c}\in[0,1]|\ \exists\ \mathbf{x}\in\mathrm{X}_{\sss\mathcal{U}}: \gamma_{c} = v_Q(\mathbf{x},c)\}$ be the set of unique votes for the unlabeled examples to the class $c$. Let enumerate its elements such that they form an ascending order:
\[
\gamma^{(1)}_c \leq \gamma^{(2)}_c \leq \dots \leq \gamma^{(N_c)}_c,
\]
where $N_c := |\Gamma_{c}|$. 
Denote $b_{i,j}^{(t)} := \frac{1}{u_i}\sum_{\mathbf{x}\in\mathrm{X}_{\sss\mathcal{U}}} P(Y=i|X=\mathbf{x}) \I{B_Q(\mathbf{x})=j}\I{v_Q(\mathbf{x},j)=\gamma^{(t)}_j}$.
\\Then, for all $(i,j) \in \{1,\ldots,K\}^2$:
\begin{align}
R_\mathcal{U}(G_Q,i,j) &\geq K_{i,j}:=\sum_{t=1}^{N_j} b^{(t)}_{i,j}\gamma^{(t)}_j,
\label{eq:lemma:gibbs:multi}\\
R_\mathcal{U\wedge\bm{\theta}}(B_Q,i,j) &= \sum_{t=k_j+1}^{N_j} b_{i,j}^{(t)},\label{eq:lemma:bayes:multi}
\end{align}
where
$k_j = \max\{t|\gamma^{(t)}_j< \theta_j\}$ with $\max(\emptyset)=0$ by convention.
\end{lem}
The proof is provided in Appendix \ref{AppendixProofLemma}. 
Following Lemma \ref{lem:connection-Gibbs-Bayes-multi}, we derive a bound on the Bayes conditional risk using the class vote distribution.
\begin{thm}
\label{thm:tr-bound-bayes-multi}
Let $B_Q$  be the $Q$-weighted majority vote classifier defined by Eq. \eqref{eq:Bayes-classifier-multi-again}.
Then for any $Q$, for all $\boldsymbol{\theta}\in[0,1]^K$, for all $(i,j) \in \{1,\ldots,K\}^2$ we have:

\begin{align}
R_{\mathcal{U}\wedge\bm{\theta}}(B_Q,i,j) &\leq \inf_{\gamma\in[\theta_j,1]}\left\{I^{(\leq,<)}_{i,j}(\theta_j, \gamma) + \frac{1}{\gamma}\floor*{K_{i,j}-M_{i,j}^<(\gamma)+M_{i,j}^<(\theta_j)}_+\right\},\label{eq:tr-bound-joint-bayes-multi}\tag{TB$_{i,j}$}
\end{align}
where  
\begin{itemize}
    \item $K_{i,j} = \frac{1}{u_i} \sum_{\mathbf{x}\in X_\mathcal{U}} P(Y=i|X=\mathbf{x})v_Q(\mathbf{x},j)\I{B_Q(\mathbf{x})=j}$ is the transductive Gibbs conditional risk evaluated on the examples for which the majority vote class is $j$,
    \item $I^{(\triangleleft_1,\triangleleft_2)}_{i,j}(s_1, s_2) = \frac{1}{u_i}\sum_{\mathbf{x}\in\mathrm{X}_{\sss\mathcal{U}}}P(Y=i|X=\mathbf{x})\I{s_1\triangleleft_1 v_Q(\mathbf{x},j) \triangleleft_2 s_2}, (\triangleleft_1,\triangleleft_2)\in\{<,\leq\}^2$ is the expected proportion of unlabeled examples from the class $i$ with $v_Q(\mbf{x}, j)$ in interval $[\theta_j, \gamma)$, 
    \item $M_{i,j}^<(s) = \frac{1}{u_i} \sum_{\mathbf{x}\in\mathrm{X}_{\sss\mathcal{U}}}P(Y=i|X=\mathbf{x})v_Q(\mathbf{x},j)\I{v_Q(\mathbf{x},j)<s}$, is the average of $j$-votes in the class $i$ that less than $s$.
\end{itemize}
\end{thm}
\begin{proof}

We would like to find an upper bound for the joint Bayes conditional risk. Hence, for all $(i,j) \in \{1,\ldots,K\}^2$, for all $\bm{\theta}\in[0,1]^K$, we consider the case when the mistake is maximized. Then, using Lemma \ref{lem:connection-Gibbs-Bayes-multi}:
\begin{equation}
\label{eq:R_U_t_B_i_j}
R_\mathcal{U\wedge\bm{\theta}}(B_Q,i,j) = \sum_{t=k_j+1}^{N_j} b_{i,j}^{(t)} \leq \max_{b^{(1)}_{i,j},\dots, b^{(N_j)}_{i,j}} \sum_{t=k_j+1}^{N_j} b_{i,j}^{(t)},
\end{equation}
with $k_j=\max\{t|\gamma^{(t)}_j< \theta_j\}\I{\{t|\gamma^{(t)}_j< \theta_j\} \not= \emptyset}$. 

Let $B^{(t)}_{i,j}=\sum_{\mathbf{x}\in\mathrm{X}_{\sss\mathcal{U}}}P(Y\!=\!i|X\!=\!\mathbf{x})\I{v_Q(\mathbf{x},j)=\gamma^{(t)}_j}/u_i$. Then, it can be noticed that
$0 \leq b_{i,j}^{(t)} \leq B^{(t)}_{i,j}$.
Remember that $K_{i,j}$ can also be written as $\sum_{t=1}^{N_j} b^{(t)}_{i,j}\gamma^{(t)}_j$.
Hence  the bound defined by Eq. \eqref{eq:R_U_t_B_i_j} should satisfy the following linear program :
\begin{align}
\max_{b^{(1)}_{i,j},\dots, b^{(N_j)}_{i,j}}\ &\sum_{t=k_j+1}^{N_j} b^{(t)}_{i,j}\label{eq:linear-program-multi}\\
\text{s.t. }\ \ \ &\forall t,\ 0 \leq b_{i,j}^{(t)} \leq B^{(t)}_{i,j}
\text{ and }\sum_{t=1}^{N_j} b^{(t)}_{i,j}\gamma^{(t)}_j= K_{i,j}\nonumber.
\end{align}
The solution of \eqref{eq:linear-program-multi} can be solved analytically and it is attained for:
\begin{equation}
\label{eq:th:5}
    b^{(t)}_{i,j}=
        \min\left(B^{(t)}_{i,j},\floor*{\frac{1}{\gamma^{(t)}_j}(K_{i,j} -\sum_{k<w<t} \gamma^{(w)}_j B^{(w)}_{i,j})}_+\right)\I{t\leq k_j}.
\end{equation}
For the sake of a better presentation, the proof of this solution is deferred to the appendix  \ref{AppendixProofLemma}, Lemma~\ref{lem:sol-lin-prog}.
Further, we can notice that, for all $(i,j) \in \{1,\ldots,K\}^2$,
$$\sum_{k_j<w<t} \gamma^{(w)}_j B^{(w)}_{i,j}=M_{i,j}^<(\gamma_{j}^{(t)})-M_{i,j}^<(\theta_j).$$ 
Let $p=\max\{t|K_{i,j}-M_{i,j}^<(\gamma_j^{(t)})+M_{i,j}^<(\theta_j)>0\}$. Then, Eq. \eqref{eq:th:5} can be re-written as follows:
\begin{equation}
\label{eq:th:6}
b^{(t)}_{i,j} = 
\begin{cases}
    0 & t\leq k_j\\
    B^{(t)}_{i,j} & k_j+1\leq t< p\\
    \frac{1}{\gamma_j^{(p)}}(K_{i,j}-M_{i,j}^<(\gamma_j^{(p)})+M_{i,j}^<(\theta_j)) & t=p\\
    0 & t>p.
\end{cases}    
\end{equation}
Notice that $\sum_{t=k_j+1}^{p-1} B_{i,j}^{(t)} = I_{i,j}^{(\leq,<)}(\theta_j, \gamma_j^{(p)})$. Using this fact as well as Eq. \eqref{eq:th:6}, we infer:
\[
 R_\mathcal{U\wedge\bm{\theta}}(B_Q,i,j) \leq I_{i,j}^{(\leq,<)}(\theta_j, \gamma_j^{(p)}) + \frac{1}{\gamma_j^{(p)}}(K_{i,j}-M_{i,j}^<(\gamma_j^{(p)})+M_{i,j}^<(\theta_j)).
\]
Consider the function 
$$\gamma \mapsto U_{i,j}(\gamma) := I_{i,j}^{(\leq,<)}(\theta_j, \gamma) + \frac{1}{\gamma}\floor*{K_{i,j}-M_{i,j}^<(\gamma)+M_{i,j}^<(\theta_j)}_+.$$
To prove the theorem, it remains to verify that, for all $(i,j) \in \{1,\ldots,K\}^2$, for all $\gamma \in[\theta_j,1],\ U_{i,j}(\gamma_j^{(p)})\leq U_{i,j}(\gamma)$. For this, consider $\gamma_j^{(w)}$ with  $w\in\{1,\dots,N_j\}$.

If $w > p$, then $U_{i,j}(\gamma_j^{(p)})\leq I_{i,j}^{(\leq,\leq)}(\theta_j, \gamma)\leq U_{i,j}(\gamma_j^{(w)}).$


If $w < p$, then
\begin{align*}
     U_{i,j}(\gamma_j^{(p)}) - U_{i,j}(\gamma_j^{(w)})=& \sum_{t=w}^{p} b^{(t)}_{i,j} - \frac{1}{\gamma_j^{(w)}}\left(K_{i,j}-M_{i,j}^<(\gamma_j^{(w)})+M_{i,j}^<(\theta_j)\right)\\
     =& \sum_{t=w}^{p} b^{(t)}_{i,j} - \frac{1}{\gamma_j^{(w)}}\left(\sum_{t=k+1}^{p} b^{(t)}_{i,j}\gamma_j^{(t)} - \sum_{t=k+1}^{w-1} \gamma^{(t)}_j b^{(t)}_{i,j}\right) \\
     =& \frac{1}{\gamma_j^{(w)}}\left(\sum_{t=w}^{p} b^{(t)}_{i,j}\gamma_j^{(w)} - \sum_{t=w}^{p} b^{(t)}_{i,j}\gamma_j^{(t)}\right)\leq 0.
\end{align*}
which completes the proof.
\end{proof}
Following this result, a transductive bound for the joint Bayes conditional risk can be found by arranging the class votes in an ascending order and considering  the linear program \eqref{eq:linear-program-multi}, where the connection with the Gibbs classifier is used as a linear constraint. Furthermore, as the bound is the infimum of the function $U_{i,j}$ on the interval $[\theta_j,1]$ it can be  computed in practice without solving the linear program explicitly.

When $\theta_j=0$, a bound over the transductive Bayes conditional risk is directly obtained from \eqref{eq:tr-bound-joint-bayes-multi} by noticing that $M_{i,j}^<(0) = 0$ in this case:
\begin{equation}
\label{eq:tr-bound-bayes-multi}
R_\mathcal{U}(B_Q,i,j)\leq \inf_{\gamma\in[0,1]}\left\{I^{(\leq,<)}_{i,j}(0, \gamma) + \frac{1}{\gamma}\floor*{K_{i,j}-M_{i,j}^<(\gamma)}_+\right\}.
\end{equation}

We note that in the binary case \citep{Amini:2008}, the transductive Gibbs risk used inside the linear program can be bounded either by the PAC-Bayesian bound \citep{Derbeko:2004,Begin:2014} or by 1/2 (the worst possible error of the binary classifier), which allows to compute the transductive bound. 
In the multi-class case, the bound can be evaluated only by approximating the posterior probabilities.   Once we estimate the posterior probability, $K_{i,j}$ and the transductive conditional Gibbs risk are also directly approximated.

\subsection{Transductive confusion matrix and transductive error rate}

In this section, based on Theorem \ref{thm:tr-bound-bayes-multi}, we derive bounds for two other error measures: the \textit{error rate} and the \textit{confusion matrix} \citep{Morvant:2012:ICML}. 
We define the transductive error rate and 
the transductive \emph{joint} error rate of the Bayes classifier $B_Q$  over the unlabeled set $\mathrm{X}_{\sss\mathcal{U}}$ given a vector $\bm{\theta} = (\theta_c)_{c=1}^K\in [0,1]^K$, as:
\begin{align}
R_{\sss\mathcal{U}}(B_Q) &:= \frac{1}{u}\sum_{\mathbf{x}\in \mathrm{X}_{\sss\mathcal{U}}}\sum_{\substack{{c\in\{1,\dots,K\}}\\{c\neq B_Q(\mbf{x})}}}P(Y=c|\mbf{X}=\mbf{x}), \nonumber\\
R_{\mathcal{U}\wedge\bm{\theta}}(B_Q) &:= \frac{1}{u}
\sum_{\mathbf{x}\in \mathrm{X}_{\sss\mathcal{U}}}
\sum_{\substack{{c\in\{1,\dots,K\}}\\{c\neq B_Q(\mbf{x})}}} P(Y=c|X=\mathbf{x})\I{v_Q(\mathbf{x}, B_Q(\mathbf{x}))\geq \theta_{B_Q(\mathbf{x})}}. \label{def:joint-bayes-err}
\end{align}


Then, we define the \emph{transductive joint Bayes confusion matrix} for $\bm{\theta}\in[0,1]^K$, and $(i,j)\in\{1,\dots,K\}^2$, as follows:
\begin{equation*}
\left[\mathbf{C}^\mathcal{U\wedge\bm{\theta}}_{h}\right]_{i,j} = \begin{cases}0 & i=j,\\ R_{\mathcal{U}\wedge\bm{\theta}}(h,i,j) & i\not=j.\end{cases}
\end{equation*}
The following proposition links the error rate with the joint confusion matrix:

\begin{prop}
\label{rmk:joint-err-connection-conf-matrix}
Let $B_Q$ be the majority vote classifier. Given a vector $\bm{\theta}\in [0,1]^K$,  for $\mathbf{p} := \{u_i/u\}_{i=1}^K$, where $u_i = \sum_{\mathbf{x}\in \mathrm{X}_{\sss\mathcal{U}}} P(Y=i|X=\mathbf{x})$, we have:
\begin{align}
R_{\mathcal{U}\wedge\bm{\theta}}(B_Q) = \norm{\T{\left(\mathbf{C}^\mathcal{U\wedge\bm{\theta}}_{B_Q}\right)}\, \mathbf{p}}_1.
\label{errorConfMatrixJoint}
\end{align}
\end{prop}
\begin{proof}

To prove Eq. \eqref{errorConfMatrixJoint}, combine the definition of transductive joint Bayes conditional risk given in Eq. \eqref{def:joint-cond-risk} and Eq. \eqref{def:joint-bayes-err} as follows: 
\begin{align*}
R_{\mathcal{U}\wedge\bm{\theta}}(B_Q) 
&= \frac{1}{u} \sum_{i=1}^K \sum_{\substack{j=1\\j\not= i}}^K \sum_{\mathbf{x}\in \mathrm{X}_{\sss\mathcal{U}}} P(Y=i|X=\mathbf{x})\I{B_Q(\mathbf{x})=j}\I{v_Q(\mathbf{x},j)\geq \theta_j} \\
&= \sum_{i=1}^K \frac{u_i}{u} \sum_{\substack{j=1\\j\not= i}}^K R_{\mathcal{U}\wedge\bm{\theta}}(B_Q,i,j)
= \norm{\T{\left(\mathbf{C}^\mathcal{U\wedge\bm{\theta}}_{B_Q}\right)}\, \mathbf{p}}_1.
\end{align*}
\end{proof}


From Theorem \ref{thm:tr-bound-bayes-multi}, we derive corresponding transductive bounds for the confusion matrix norm and the error rate of the Bayes classifier.
To simplify notations, we introduce a matrix $\mathbf{U}_{\bm{\theta}}$ of size $K\times K$ with zeros on the main diagonal and the following $ (i,j)$-entries, $i \neq j$:
\[
\left[\mathbf{U}_{\bm{\theta}}\right]_{i,j}:= \inf_{\gamma\in[\theta_j,1]}\left\{I^{(\leq,<)}_{i,j}(\theta_j, \gamma) + \frac{1}{\gamma}\floor*{(K_{i,j}-M_{i,j}^<(\gamma)+M_{i,j}^<(\theta_j))}_+\right\},
\]
which corresponds to the transductive bound proposed in Theorem \ref{thm:tr-bound-bayes-multi}.
\begin{cor}
\label{cor:matrix-bound}
For all $\boldsymbol{\theta}\in[0,1]^K$, we have:
\begin{equation}
\label{eq:cor-conf}
\mnorm{\mathbf{C}^{\mathcal{U}\wedge\bm{\theta}}_{B_Q}} \leq \mnorm{\mathbf{U}_{\bm{\theta}}}.
\end{equation}
Moreover, we have the following bound:
\begin{equation}
\label{eq:cor-err}
R_{\sss\mathcal{U}\wedge\bm{\theta}}(B_Q) \leq \norm{\T{\mathbf{U}_{\bm{\theta}}}\, \mathbf{p}}_1.
\end{equation}
where $\|.\|$ is the spectral norm; and $\mathbf{p} = \{u_i/u\}_{i=1}^K$, with $u_i = \sum_{\mathbf{x}\in \mathrm{X}_{\sss\mathcal{U}}} P(Y=i|X=\mathbf{x})$.
\end{cor}

\begin{proof}
                                                                                                  

The confusion matrix $\mathbf{C}^{\mathcal{U}\wedge\bm{\theta}}_{B_Q}$ is always non-negative, and from Theorem \ref{thm:tr-bound-bayes-multi}, each of its entries is smaller than the corresponding entry of $\mathbf{U}_{\bm{\theta}}$. Hence, from the property of spectral norm for two positive matrices $\mathbf{A}$ and $\mathbf{B}$~:
\[
\mathbf{0}_{K,K}\preceq \mathbf{A} \preceq \mathbf{B} \Rightarrow \|\mathbf{A}\|\leq \|\mathbf{B}\|,
\]
where $\mathbf{A} \preceq \mathbf{B}$ denotes that each element of $\mathbf{A}$ is smaller than the corresponding element of $\mathbf{B}$, we deduce Eq. \eqref{eq:cor-conf}.

With the same computations, we observe the following inequality:
\[
\T{\left(\mathbf{C}^{\mathcal{U}\wedge\bm{\theta}}_{B_Q}\right)}\, \mathbf{p} \leq \T{\mathbf{U}_{\bm{\theta}}}\, \mathbf{p}.
\]
Elements of the left vector are non-negative. Hence the inequality holds for the $\ell_1$-norm, and taking into account Proposition \ref{rmk:joint-err-connection-conf-matrix} we infer:
\[
R_{\mathcal{U}\wedge\bm{\theta}}(B_Q) = \norm{\T{\left(\mathbf{C}^{\mathcal{U}\wedge\bm{\theta}}_{B_Q}\right)}\, \mathbf{p}}_1 \leq \norm{\T{\mathbf{U}_{\bm{\theta}}}\, \mathbf{p}}_1.
\]
\end{proof}

Note that the transductive bound of the Bayes error rate is obtained from Eq. \eqref{eq:cor-err} by taking $\bm\theta$ as the zero vector $\mathbf{0}_K$:
\begin{equation}
    \label{eq:TB}\tag{TB}
    R_{\mathcal{U}}(B_Q) \leq \norm{\T{\mathbf{U}_{\mbf{0}_K}}\, \mathbf{p}}_1.
\end{equation}

\subsection{Tightness Guarantees}
In this section, we assume that the Bayes classifier makes most of its error on unlabeled examples with a low prediction vote, i.e., class votes can be considered as indicators of confidence. In the following proposition, we show that the bound becomes tight under certain conditions. We remind that $\Gamma_{j}=\{\gamma_j^{(t)}\}$ is the set of unique votes for the unlabeled examples to the class $j$, and $b_{i,j}^{(t)}$ corresponds to the Bayes conditional risk on the examples with the vote $\gamma_j^{(t)}$ (see Lemma \ref{lem:connection-Gibbs-Bayes-multi} for more details).

\begin{prop}
\label{prop:tight-bayes-multi}
Let $\Gamma_j^\tau:=\{\gamma_j^{(t)}\in\Gamma_j|b^{(t)}_{i,j} > \tau\}$, where $\tau\in[0,1]$ is a given threshold.
If there exists a lower bound $C\in[0,1]$ such that for all $\gamma\in\Gamma_j^\tau$:
\begin{align}
\sum_{\mathbf{x}\in\mathrm{X}_{\sss\mathcal{U}}}P(Y=i|X=\mathbf{x})\I{B_Q(\mathbf{x})=j}\I{v_Q(\mathbf{x},j)<\gamma} &\geq C \sum_{\mathbf{x}\in\mathrm{X}_{\sss\mathcal{U}}}P(Y=i|X=\mathbf{x})\I{v_Q(\mathbf{x},j)<\gamma}, \label{prop-multi:initial-cond}
\end{align}
then, the following inequality holds:
\begin{equation*}
    \left[\mathbf{U}_{\mathbf{0}_K}\right]_{i,j} - R_\mathcal{U}(B_Q,i,j) \leq \frac{1-C}{C}R_{\mathcal{U}}(B_Q,i,j) + r_{i,j}\left(\frac{1}{\gamma^*_j}-1\right),
\end{equation*}
where \begin{itemize}
    \item $\gamma^*_j := \sup\{\gamma_j^{(t)}\in\Gamma_j^\tau\}$ is the highest vote which satisfies $b^{(t)}_{i,j} > \tau$, and
    \item $r_{i,j} := \sum_{\mathbf{x}\in\mathrm{X}_{\sss\mathcal{U}}}P(Y=i|X=\mathbf{x})v_Q(\mathbf{x},j)\I{B_Q(\mathbf{x})=j}\I{v_Q(\mathbf{x},j)>\gamma^*_j}/u_i$ corresponds to the average of $j$-votes in the class $i$ that greater than $\gamma^*_j$ and on which the Bayes classifier makes the conditional mistake.
\end{itemize}
\end{prop}

\begin{proof}

First, it can be proved that for all $\mathbf{x}\in\mathrm{X}_{\sss\mathcal{U}}$, for all $(i,j)\in \{1,\ldots,K\}^2$, the following inequality holds:
\begin{multline}
\label{eq:prop-multi:1}
R_\mathcal{U}(B_Q,i,j) \geq \frac{1}{u_i}\sum_{\mathbf{x}\in\mathrm{X}_{\sss\mathcal{U}}}P(Y=i|X=\mathbf{x})\I{B_Q(\mathbf{x})=j}\I{v_Q(\mathbf{x},j)<\gamma^*} \\+ \frac{1}{\gamma^*}\floor*{\floor{K_{i,j}-M_{i,j}^<(\gamma^*)}_+ - r_{i,j}}_+ + r_{i,j},
\end{multline}
where $\gamma^* := \sup\{\gamma\in\Gamma_j|\sum_{\mathbf{x}\in\mathrm{X}_{\sss\mathcal{U}}}P(Y=i|X=\mathbf{x})\I{B_Q(\mathbf{x})=j}\I{v_Q(\mathbf{x},j)=\gamma}/u_i> \tau\}$. We prove this result in Lemma \ref{lem:lem-for-proposition} in Appendix. Now, taking into account Eq. \eqref{eq:prop-multi:1}  and Eq. \eqref{prop-multi:initial-cond} we deduce the following:
\begin{align}
R_\mathcal{U}(B_Q,i,j) \geq & \frac{C}{u_i}\sum_{\mathbf{x}\in\mathrm{X}_{\sss\mathcal{U}}}P(Y=i|X=\mathbf{x})\I{v_Q(\mathbf{x},j)<\gamma^*} + \frac{1}{\gamma^*}\floor*{\floor{K_{i,j}-M_{i,j}^<(\gamma^*)}_+ - r_{i,j}}_+ + r_{i,j} \nonumber\\ 
=& C\,I^{(\leq,<)}_{i,j}(0, \gamma^*)  + \frac{1}{\gamma^*}\floor*{\floor{K_{i,j}-M_{i,j}^<(\gamma^*)}_+ - r_{i,j}}_+ + r_{i,j}. \label{eq:prop-multi:4} 
\end{align}
By definition of $\mathbf{U}_{\mathbf{0}_K}$ we have, for all $(i,j)\in \{1,\ldots,K\}^2$,
\begin{equation}
\label{eq:prop-multi:5}
\left[\mathbf{U}_{\mathbf{0}_K}\right]_{i,j} \leq I^{(\leq,<)}_{i,j}(0, \gamma^*) + \frac{1}{\gamma^*}\floor*{K_{i,j}-M_{i,j}^<(\gamma^*)}_+.
\end{equation}
Subtracting Eq. \eqref{eq:prop-multi:4} from Eq. \eqref{eq:prop-multi:5} we obtain:
\begin{multline*}
\left[\mathbf{U}_{\mathbf{0}_K}\right]_{i,j} - R_\mathcal{U}(B_Q,i,j) \leq (1-C)I^{(\leq,<)}_{i,j}(0, \gamma^*) \\+ \frac{1}{\gamma^*}\left(\floor*{K_{i,j}-M_{i,j}^<(\gamma^*)}_+ - \floor*{\floor{K_{i,j}-M_{i,j}^<(\gamma^*)}_+ - r_{i,j}}_+\right)-r_{i,j}.
\end{multline*}
We can notice that for all $a,b\in\mathbb{R}^+:\ b-\floor{b-a}_+\leq a$. Then, we have:
\begin{equation}
\label{eq:prop-multi:6}
\left[\mathbf{U}_{\mathbf{0}_K}\right]_{i,j} - R_\mathcal{U}(B_Q,i,j) \leq (1-C)I^{(\leq,<)}_{i,j}(0, \gamma^*) + r_{i,j}\left(\frac{1}{\gamma^*}-1\right).
\end{equation}
Also, from Eq. \eqref{eq:prop-multi:4} one can derive: 
\begin{align}
I^{(\leq,<)}_{i,j}(0, \gamma^*)&\leq \frac{1}{C}\left(R_\mathcal{U}(B_Q,i,j) - \frac{1}{\gamma^*}\floor*{\floor{K_{i,j}-M_{i,j}^<(\gamma^*)}_+ - r_{i,j}}_+ - r_{i,j} \right)
\leq \frac{R_\mathcal{U}(B_Q,i,j)}{C}\label{eq:prop-multi:7}.
\end{align}

Taking into account Eq. \eqref{eq:prop-multi:6} and Eq. \eqref{eq:prop-multi:7}, we infer:
\[
\left[\mathbf{U}_{\mathbf{0}_K}\right]_{i,j} - R_\mathcal{U}(B_Q,i,j) \leq \frac{1-C}{C} R_\mathcal{U}(B_Q,i,j) + r_{i,j}\left(\frac{1}{\gamma^*}-1\right).
\]
\end{proof}

This proposition states that if Eq. \eqref{prop-multi:initial-cond} holds, the difference between the transductive Bayes conditional risk and its upper bound does not exceed an expression that depends on a constant $C$ and a threshold $\tau$. When the majority vote classifier makes most of its mistake for the class $j$ on observations with a low value of $v_Q(\mathbf{x}, j)$, with a reasonable choice of $\tau$, $r_{i,j}$ and $\gamma^*_j$ are decreasing. This also implies that Eq. \eqref{prop-multi:initial-cond} accepts a high value $C$ (close to 1) and the bound will be tighter. The closer our framework to the deterministic one, the closer $r_{i,j}$ will be to 0 ( in the deterministic case, $\tau$ can be set to 0, so $r_{i,j}$ will be 0), so the bound becomes tight. Although our bound is tight only under the condition of making mistakes on low prediction votes, the assumption is reasonable from the theoretical point of view, since if for some observation the Bayes classifier gives a relatively high vote to the class $j$, we expect that the observation is most probably from this class and not from the class $i$. From the practical point of view, this assumption requires the learning model to be well calibrated \citep{Gebel:2009}.

\subsection{Multi-class Self-learning Algorithm}
\label{sec:msla}
In this section, we describe an application of results obtained in Section \ref{sec:tr-bound-theory} for learning on partially-labeled data. For this, we consider a self-learning algorithm \citep{Amini:15}, which is a semi-supervised approach that performs augmentation of the labeled set by pseudo-labeling unlabeled examples.

The algorithm starts from a supervised base classifier initially trained on available labeled examples. Then, it iteratively assigns pseudo-labels at each iteration to those unlabeled examples that have a confidence score above a certain threshold. The pseudo-labeled examples are then included in the training set, and the classifier is retrained. The process is repeated until no examples for pseudo-labeling are left.

The central question of applying the self-learning algorithm in practice is how to choose the threshold. Intuitively, the threshold can manually be set to a very high value, since only examples with a very high degree of confidence will be pseudo-labeled in this case. However, the confidence measure is biased by the small labeled set, so every iteration of the self-learning may still induce an error and shift the boundary in the wrong direction. In addition, the fact that a large number of iterations makes the algorithm computationally expensive drives us to choose the threshold carefully.

{To overcome this problem, we extend the strategy proposed by \citet{Amini:2008} to the multi-class setting. We consider the majority vote as the base classifier and the prediction vote as an indicator of confidence. Given a threshold vector $\bm{\theta}$,} we introduce the \emph{conditional Bayes error rate} $R_\mathcal{U|\boldsymbol{\theta}}(B_Q)$, defined in the following way:
\begin{equation}
\label{eq:cond-bayes-error}
R_{\mathcal{U}|\boldsymbol{\theta}}(B_Q) := \frac{R_{\mathcal{U}\wedge\boldsymbol{\theta}}(B_Q)}{\pi(v_Q(\mathbf{x},k)\geq\theta_k)},
\end{equation}
where $\pi(v_Q(\mathbf{x},k)\geq\theta_k):= \sum_{\mathbf{x}\in\mathrm{X}_{\sss\mathcal{U}}}\mathds{1}_{v_Q(\mathbf{x},k)\geq\theta_k}/u$ and $k := B_Q(\mathbf{x})$. The numerator reflects the proportion of mistakes on the unlabeled set when the threshold is equal to $\boldsymbol{\theta}$, whereas the denominator computes the proportion of unlabeled observations with the vote larger than the threshold for the predicted class. {Thus, we propose to find the threshold that yields the minimal value of $R_\mathcal{U|\boldsymbol{\theta}}(B_Q)$,  making a trade-off between the error we induce by pseudo-labeling and the number of pseudo-labeled examples. In Algorithm \ref{alg:MSLA} we summarize our algorithm, which is further denoted by \texttt{MSLA}\footnote{The code source of the algorithm can be found at \url{https://github.com/vfeofanov/trans-bounds-maj-vote}.}.
}

To evaluate the transductive error, we bound the numerator of Eq. \eqref{eq:cond-bayes-error} by Corollary \ref{cor:matrix-bound}. However, the bound can practically be computed only with assumptions, since the posterior probabilities $P(Y=c|X=\mathbf{x})$ for unlabeled examples are not known. In this work, we approximate the posterior $P(Y=c|X=\mathbf{x})$ by $v_Q(\mathbf{x},c)$ of the base classifier trained on labeled examples only (the initial step of \texttt{MSLA}). Although this approximation is optimistic, by formulating the bound as probabilistic we keep some chances for other classes so the error of the supervised classifier can be smoothed. However, it must be borne in mind that the hypothesis space should be diverse enough so that the entropy of $(v_Q(\mbf{x}, c))_{c=1}^K$ would not be always zero, and the errors are made mostly on low prediction votes. In our experiments, as the base classifier we use the random forest \citep{Breiman:2001} that aggregates predictions from trees learned on different bootstrap samples. In Appendix \ref{sec:exp-prob-estim}, we validate the proposed approximation by empirically comparing it with the case when the posterior probabilities are set to $1/K$, i.e., when we treat all classes as equally probable. 



\begin{algorithm}[ht!]
\caption{Multi-class self-learning algorithm\,(MSLA)}
\label{alg:MSLA}
\begin{algorithmic}
\State
\State \textbf{Input:} \\ Labeled observations $\mathrm{Z}_{\mathcal{L}}$ 
                      \\ Unlabeled observations $\mathrm{X}_{\sss\mathcal{U}}$
\State \textbf{Initialisation:} 
\\A set of pseudo-labeled instances, $\mathrm{Z}_{\mathcal{P}}\leftarrow \emptyset$
\\A classifier $B_Q$ trained on $\mathrm{Z}_{\mathcal{L}}$
\Repeat
    \State \textbf{1.} Compute the vote threshold $\bm{\theta^*}$ that minimizes the conditional Bayes error rate: 
        \begin{equation*}
        \bm{\theta}^* = \argmin_{\bm{\theta}\in(0,1]^K} R_{\mathcal{U}|\bm{\theta}}(B_Q) .\tag{$\star$}
        \end{equation*}
    \State \textbf{2.} $S \leftarrow\{(\mathbf{x},y')|\mathbf{x}\in\mathrm{X}_{\sss\mathcal{U}};[v_Q(\mathbf{x},y')\geq\theta^*_{y'}]\wedge [ y'= \argmax_{c\in}v_Q(\mathbf{x},c) ]\}$
    \State \textbf{3.} $\mathrm{Z}_{\mathcal{P}}\leftarrow \mathrm{Z}_{\mathcal{P}} \cup S$, $\mathrm{X}_{\sss\mathcal{U}} \leftarrow \mathrm{X}_{\sss\mathcal{U}}\setminus S$
    \State \textbf{4.} Learn a classifier $B_Q$ with the following loss function:
    \[
    \mathcal{L}(B_Q,Z_{\mathcal{L}},\mathrm{Z}_{\mathcal{P}}) = \frac{l+|\mathrm{Z}_{\mathcal{P}}|}{l}\mathcal{L}(B_Q,Z_{\mathcal{L}}) + \frac{l+|\mathrm{Z}_{\mathcal{P}}|}{|\mathrm{Z}_{\mathcal{P}}|}\mathcal{L}(B_Q,\mathrm{Z}_{\mathcal{P}})
    \]
\Until{$\mathrm{X}_{\sss\mathcal{U}}\text{ or }S \text{ are }\emptyset$}
\State \textbf{Output:} The final classifier $B_Q$
\end{algorithmic}
\end{algorithm}

Similarly to the work of \citet{Amini:2008}, in practice, to find an optimal $\bm{\theta}^*$ we perform a grid search over the hypercube $(0,1]^K$. The same algorithm is used for computing the optimal $\gamma^*$ that provides the value of an upper bound for the conditional risk (see Theorem \ref{thm:tr-bound-bayes-multi}). In contrast to the binary self-learning, the direct grid search in the multi-class setting costs $O\left(R^K\right)$, where $R$ is the sampling rate of the grid. As
\begin{align*}
R_{\mathcal{U}|\boldsymbol{\theta}}(B_Q) 
&= \sum_{j=1}^K\frac{R_{\mathcal{U}\wedge\boldsymbol{\theta}}^{(j)}(B_Q)}{\sum_{c=1}^K\frac{1}{u}\sum_{\mathbf{x}\in\mathrm{X}_{\sss\mathcal{U}}}\mathds{1}_{v_Q(\mathbf{x},c)\geq\theta_c}\mathds{1}_{B_Q(\mathbf{x})=c}} 
\leq \sum_{j=1}^K\frac{R_{\mathcal{U}\wedge\boldsymbol{\theta}}^{(j)}(B_Q)}{\frac{1}{u}\sum_{\mathbf{x}\in\mathrm{X}_{\sss\mathcal{U}}}\mathds{1}_{v_Q(\mathbf{x},j)\geq\theta_j}\mathds{1}_{B_Q(\mathbf{x})=j}}\nonumber\\
&\leq \sum_{j=1}^K\frac{R_{\mathcal{U}\wedge\boldsymbol{\theta}}^{(j)}(B_Q)}{\pi\{(v_Q(\mathbf{x},j)\geq\theta_j)\land (B_Q(\mathbf{x})=j)\}}\label{eq:paralell-upp-bound} \tag{$\ast$},
\end{align*}
where $R_{\mathcal{U}\wedge\boldsymbol{\theta}}^{(j)}(B_Q)=\sum_{i=1}^K u_iR_{\mathcal{U}\wedge\boldsymbol{\theta}}(B_Q,i,j)/u$, the sum might be minimized term by term, tuning independently each component of $\boldsymbol{\theta}$. This replaces the $K$-dimensional minimization task  by $K$ tasks of 1-dimensional minimization.



\section{Probabilistic C-Bound with Imperfect Labels}
\label{sec:c-bound}
The transductive bound \eqref{eq:TB} can be regarded as a first-order bound, since it is linearly dependent on the classifier' votes, so it does not take into account the correlation between hypotheses.
In addition, despite its application for minimization of the error induced by self-learning, the obtained pseudo-labels may be still erroneous, and we do not know how to evaluate the classification error in this noisy case. In this section, we overcome these two issues by deriving a new probabilistic C-bound in the presence of imperfect labels.

\subsection{C-Bound in the Probabilistic Setting}
\citet{Lacasse:2007} proposed to upper bound the Bayes error by taking into account the mean and the variance of the prediction margin, which, we recall Eq.~\eqref{def:margin}, is defined as $v_Q(\mbf{x},y) - \max_{\substack{{c\in\mathcal{Y}}\setminus\{y\}}} v_Q(\mbf{x},c)$. A similar result was obtained in a different context by \citet{Breiman:2001}. \citet{Laviolette:2017} extended this bound to the multi-class case. 

Below, we derive their C-bound in the probabilistic setting. Now, we consider the \emph{generalization error} as an error measure, which is defined in the probabilistic setting
as follows:
\begin{align*}
R(B_Q) &
:=\E_{P(\mbf{X})}\sum_{\substack{{c\in\{1,\dots,K\}}\\{c\neq B_Q(\mbf{x})}}}P(Y=c|\mbf{X}=\mbf{x})
=\E_{P(\mbf{X})}[1-P(Y=B_Q(\mbf{x})|\mbf{X}=\mbf{x})].
\end{align*}

\begin{thm}
\label{thm:prob-cbound}
Let $M$ be a random variable such that $[M|\mbf{X}=\mbf{x}]$ is a discrete random variable that is equal to the margin $M_Q(\mbf{x}, c)$ with probability $P(Y=c|\mbf{X}=\mbf{x})$, $c=\{1,\dots,K\}$.
Let $\mu^{M}_1$ and $\mu^{M}_2$ be the first and the second statistical moments of the random variable $M$, respectively.
Then, for all choice of $Q$ on a hypothesis space $\mathcal{H}$, and for all distributions $P(\mbf{X})$ over $\mathcal{X}$ and $P(Y|\mbf{X})$ over $\mathcal{Y}$, such that $\mu^M_1>0$, we have:
\begin{align}
    \label{eq:prob-cbound}\tag{CB}
    R(B_Q) \leq 1 - \frac{(\mu^M_1)^2}{\mu^M_2}.
\end{align}
\end{thm}
\begin{proof}
    At first, we show that $R(B_Q) = P(M\leq 0)$.
    For a fixed $\mbf{x}$, one get:
    \begin{align*}
        P(M\leq 0|\mbf{X}=\mbf{x}) = \sum_{c=1}^K P(Y=c|\mbf{X}=\mbf{x})\I{M_Q(\mbf{x},c)\leq 0} = \sum_{\substack{{c\in\{1,\dots,K\}}\\{c\neq B_Q(\mbf{x})}}}P(Y=c|\mbf{X}=\mbf{x}).  
    \end{align*}

    Applying the total probability law, we obtain:
    \begin{align}
    P(M\leq 0) &= \int_{\mathcal{X}} P(M\leq 0|\mbf{X}=\mbf{x}) P(\mbf{X}=\mbf{x})\diff\mbf{x}= \E_{P(\mbf{X})} P(M\leq 0|\mbf{X}=\mbf{x}) = R(B_Q). \label{eq:bayes-risk-via-prob-margins}
    \end{align}
    By applying the Cantelli-Chebyshev inequality (Lemma \ref{lem:cantelli-chebyshev} in Appendix), we deduce:
    
    \begin{align}
       P(M\leq 0) &\leq \frac{\mu^M_2 - (\mu^M_1)^2}{\mu^M_2 - (\mu^M_1)^2 + (\mu^M_1)^2 } = 1 - \frac{(\mu^M_1)^2}{\mu^M_2}. \label{eq:prob-M-less-0-bound}
    \end{align}
    Combining Eq. \eqref{eq:bayes-risk-via-prob-margins} and Eq. \eqref{eq:prob-M-less-0-bound} gives the bound.
\end{proof}

Thus, the probabilistic C-bound allows to bound the generalization error of the Bayes classifier when examples are provided with probabilistic labels. Note that when for every example, only one label is possible, the bound comes back to the usual deterministic case.

The main advantage of C-bound is the involvement of the second margin moment, which can be related to correlations between hypotheses' predictions \citep{Lacasse:2007}. 




\subsection{Mislabeling Error Model}
\label{sec:mislab-error-model}
The self-learning algorithm, which was introduced in Section \ref{sec:msla}, supplies the unlabeled examples with pseudo-labels that are potentially erroneous. In this section, we consider a mislabeling error model to explicitly take into account this issue.

We consider an imperfect output $\hat{Y}$, which has a different distribution from the true output $Y$. The label imperfection is summarized through the \emph{mislabeling matrix} $\mbf{P}=(p_{j,c})_{1\leq j,c\leq K}$, defined by:
\begin{align}
\label{eq:mislab-model}
    P(\hat Y=j|Y=c) &:= p_{j,c} \quad\forall(j,c)\in \{1,\dots,K\}^2,
\end{align}
where $\sum_{j=1}^K p_{j,c} = 1$. Additionally, we assume that $\hat Y$ does not influence the true class distribution: $P(\mbf{X}|Y, \hat Y) = P(\mbf{X}, Y)$. This implies that 
\begin{align}
    \label{eq:mislab-prob-transformation}
    {P(\hat Y=j|\mbf{X}=\mbf{x})=\sum_{c=1}^K p_{j,c}P(Y=c|\mbf{X}=\mbf{x}).}
\end{align}
{This class-related model is a common approach to deal with the label imperfection \citep{Chittineni:1980,Amini:2003,Natarajan:2013,Scott:2015}.}


At first, we derive a bound that connects the error of the true and the imperfect label in misclassifying a particular example $\mbf{x} \in \mathcal{X}$. We denote 
\begin{align*}
   r(\mbf x) &=  \sum_{\substack{{c\in\{1,\dots,K\}}\\{c\neq B_Q(\mbf{x})}}}P(Y=c|\mbf{X}=\mbf{x}), \qquad
    \hat r(\mbf x) =  \sum_{\substack{{c\in\{1,\dots,K\}}\\{c\neq B_Q(\mbf{x})}}}P(\hat Y=c|\mbf{X}=\mbf{x}).
\end{align*}
\begin{thm}
\label{thm:one-ex-risk-mislab-bound}
Let $\mbf{P}$ be  the mislabeling matrix, and assume that $p_{i,i}> p_{i,j},\ \forall{i,j}\in\{1,\dots,K\}^2$. Then, for all choice of $Q$ on a hypothesis space $\mathcal{H}$ we have, for $\mbf x \in \mathcal X$,
\begin{align}
    r(\mbf{x}) \leq \frac{\hat{r}(\mbf{x})}{\delta(\mbf{x})}-\frac{1-\alpha(\mbf{x})}{\delta(\mbf{x})}, \label{eq:one-x-mislabel-ineq}
\end{align}
with
$\delta(\mbf{x}):= p_{B_Q(\mbf{x}),B_Q(\mbf{x})}- \max_{j\in\mathcal{Y}\setminus\{B_Q(\mbf{x})\}}p_{B_Q(\mbf{x}),j}$ and $\alpha(\mbf{x}):=p_{B_Q(\mbf{x}),B_Q(\mbf{x})}$.
\end{thm}
\begin{proof}
First, from the definition of $\hat{r}(\mbf{x})$ and applying Eq. \eqref{eq:mislab-prob-transformation} we obtain that
\begin{align}
    \hat{r}(\mbf{x}) &= 1 - P(\hat Y=B_Q(\mbf{x})|\mbf X = \mbf{x}) = 1 - \sum_{j=1}^K p_{B_Q(\mbf{x}),j}P(Y=j|\mbf X = \mbf{x}) \nonumber\\
    &=1 - p_{B_Q(\mbf{x}),B_Q(\mbf{x})}P(Y=B_Q(\mbf{x})|\mbf X = \mbf{x})- \sum_{\substack{{j=1}\\{j\neq B_Q(\mbf{x})}}}^K p_{B_Q(\mbf{x}),j}P(Y=j|\mbf X = \mbf{x}) \nonumber
    \end{align}
    
    One can notice that
    \begin{align*}
        \sum_{\substack{{j=1}\\{j\neq B_Q(\mbf{x})}}}^K p_{B_Q(\mbf{x}),j}P(Y=j|\mbf X = \mbf{x}) &\leq  \max_{j\in\mathcal{Y}\setminus\{B_Q(\mbf{x})\}}p_{B_Q(\mbf{x}),j}\sum_{\substack{{j=1}\\{j\neq B_Q(\mbf{x})}}}^K P(Y=j|\mbf X = \mbf{x}) \\
        &=  \max_{j\in\mathcal{Y}\setminus\{B_Q(\mbf{x})\}}p_{B_Q(\mbf{x}),j}(1-P(Y=B_Q(\mbf{x})|\mbf X = \mbf{x})).
    \end{align*}
    
    Finally, we infer the following inequality:
   \begin{align}
    \hat{r}(\mbf{x}) &\geq (p_{B_Q(\mbf{x}),B_Q(\mbf{x})}-\max_{j\in\mathcal{Y}\setminus\{B_Q(\mbf{x})\}}p_{B_Q(\mbf{x}),j})(1-P(Y=B_Q(\mbf{x})|\mbf X = \mbf{x}))+1-p_{B_Q(\mbf{x}),B_Q(\mbf{x})} \nonumber\\
    &= \delta(\mbf{x})r(\mbf{x})+1-\alpha(\mbf{x}). \label{eq:one-x-mislabel-last-ineq-proof}
\end{align}
Taking into account the assumption that $p_{i,i}> p_{i,j},\ \forall{i,j}\in\{1,\dots,K\}^2$, we deduce that $\delta(\mbf{X})> 0$, which concludes the proof.
\end{proof}
This theorem gives us insights on how the true error rate can be bounded given the error rate of the imperfect label and the mislabeling matrix. With the quantities $\delta(\mbf{x})$ and $\alpha(\mbf{x})$, we perform a correction of $\hat{r}(\mbf{x})$. Note that when there is no mislabeling, the left and right sides of Eq. \eqref{eq:one-x-mislabel-ineq} are equal, since $\alpha(\mbf{x})=1$ and $\delta(\mbf{x})=1$ in this case. 

Note that this theorem holds also for a more general case when correction probabilities depend on the example $\mbf{x}$. In this case, all probabilities $p_{i,j}$ are replaced by $p^\mbf{x}_{i,j}:= P(\hat{Y}=i|Y=j,\mbf{X}=\mbf{x})$. Since it is harder to estimate $p^\mbf{x}_{i,j}$ compared to $p_{i,j}$, we stick to consider the class-related model described in Eq. \eqref{eq:mislab-prob-transformation}.

In the theorem, the mislabeling matrix is assumed given, while in practice it has to be estimated. Since the number of matrix entries grows quadratically with the increase of $K$, a direct estimation of the true posterior probabilities from Eq. \eqref{eq:mislab-prob-transformation} may be more affected by the estimation error than the bound itself as the latter needs to know only $2K$ entries. We give more details about estimation of the mislabeling matrix in Section \ref{sec:concl}.

The bound can be compared with a bound derived in \citet[Eq. (3.14), p. 284]{Chittineni:1980} for the optimal Bayes classifier (maximum a-posteriori rule). It is shown that $r(\mbf{x})\leq 1-\frac{1-\hat{r}(\mbf{x})}{\beta}$, where $ \beta=\max_{i=1,\dots, K}\left(\sum_{j=1}^K p_{i,j}\right)$. One can notice that the regularizer $\beta$ is constant with respect to $\mbf{x}$, so the penalization of the error rate $\hat{r}(\mbf{x})$ does not depend on the label the classifier predicts. Another limitation is that the bound assumes that the Bayes classifier is optimal.

The assumption of Theorem \ref{thm:one-ex-risk-mislab-bound} requires that the diagonal entries of the mislabeling matrix are the largest elements in their corresponding columns, which means that the imperfect label is reasonably correlated with the true label.
However, in practice, the assumption may not hold, so the theorem is not applicable. 
To overcome this, it can be relaxed by considering $\lambda>0$ such that $\lambda+\delta(\mbf{x})>0$, and so we get a bound for all choices of $Q$ on a hypothesis space $\mathcal{H}$:
\begin{align}
    r(\mbf{x}) \leq \frac{\hat{r}(\mbf{x})}{\lambda+\delta(\mbf{x})}-\frac{1-\lambda-\alpha(\mbf{x})}{\lambda+\delta(\mbf{x})}. \label{eq:one-x-mislabel-ineq-with-lam}
\end{align}
When $\delta(\mbf{x})$ is close to 0, it also avoids the bound to become arbitrarily large. 
The use of this bound is illustrated in Section \ref{sec:relax_bound} of Appendix.

\subsection{C-Bounds with Imperfect Labels}

Based on Theorem \ref{thm:one-ex-risk-mislab-bound}, we bound the generalization error  $R(B_Q)$, which is the expectation of $r(\mbf{X})$. 
By taking expectation in Eq. \eqref{eq:one-x-mislabel-ineq}, we obtain that
\begin{align}
    R(B_Q) = \E_{\mbf{X}} r(\mbf{X}) \leq \E_{\mbf{X}} \frac{\hat{r}(\mbf{X})}{\delta(\mbf{X})}-\E_{\mbf{X}}\frac{1-\alpha(\mbf{X})}{\delta(\mbf{X})}. 
    \label{eq:expect-from-one-x-risk}
\end{align}
One can see that for every $\mbf{x}$, $\hat{r}(\mbf{x})$ is multiplied by a positive weight $1/\delta(\mbf{X})>0$, so the first term of the right-hand side is a weighted generalization error  of the imperfect label. To cope with this, we derive a weighted C-bound by proposing the next theorem. 
\begin{thm}
\label{thm:w-cbound}
Let $\hat{M}$ be a random variable such that $[\hat{M}|\mbf{X}=\mbf{x}]$ is a discrete random variable that is equal to the margin $\hat{M}_Q(\mbf{x}, i)$ with probability $P(\hat{Y}=i|\mbf{X}=\mbf{x})$, $i=\{1,\dots,K\}$. Assume that every diagonal entry of the mislabeling matrix $\mbf{P}$ is the largest element in the corresponding column, i.e., $p_{i,i}> p_{i,j},\ \forall{i,j}\in\{1,\dots,K\}^2$.
Then, for all choice of $Q$ on a hypothesis space $\mathcal{H}$, and for all distributions $P(\mbf{X})$ over $\mathcal{X}$ and $P(Y|\mbf{X})$ over $\mathcal{Y}$, we have: 
\begin{align}
\label{eq:w-cbound}\tag{CBIL}
    R(B_Q)\leq \psi_{\mbf{P}} - \frac{\left(\mu_1^{\hat M,{\mbf{P}}}\right)^2}{\mu_2^{\hat M,{\mbf{P}}}},
\end{align}
if $\mu_1^{\hat M_{\mbf{P}}}>0$, where 
\begin{itemize}
    \item $\psi_{\mbf{P}}:=\E_{\mbf{X}}\frac{\alpha(\mbf{X})}{\delta(\mbf{X})} $ with $\delta$ and $\alpha$ defined as in Theorem \ref{thm:one-ex-risk-mislab-bound},
    \item $\mu_1^{\hat M,{\mbf{P}}}:=\int_{\R^{d+1}} \frac{m}{\delta(\mbf{x})} P(\hat M=m,\mbf{X}=\mbf{x})\diff\mbf{x}\diff m$ is the weighted 1st margin moment,
    \item $\mu_2^{\hat M,{\mbf{P}}}:=\int_{\R^{d+1}} \frac{m^2}{\delta(\mbf{x})} P(\hat M=m,\mbf{X}=\mbf{x})\diff\mbf{x}\diff m$ is the weighted 2nd margin moment.
\end{itemize}
\end{thm}
\begin{proof}
    At first, let us introduce a normalization factor $\omega_{\sss\mbf{P}}$ defined as follows:
    \begin{align*}
        \omega_{\sss\mbf{P}}:=\E_{\mbf{X}}\frac{1}{\delta(\mbf{X})} =\int_{\R^{d+1}} \frac{P(\hat M=m,\mbf{X}=\mbf{x})}{\delta(\mbf{x})}\diff\mbf{x}\diff m.   
    \end{align*}
    Remind that $\hat r(\mbf{x})=P(\hat{M}\leq 0|\mbf{X}=\mbf{x})$. Then, we can write:
        \begin{align}
            \E_{\mbf{X}}\frac{\hat{r}(\mbf{X})}{\delta(\mbf{X})}&=\int_{\R^d} \frac{1}{\delta(\mbf{x})}P(\hat M\leq 0|\mbf{X}=\mbf{x})P(\mbf{X}=\mbf{x})\diff \mbf{x} 
            =\int_{-\infty}^{0}\int_{\R^d} \frac{P(\hat M=m,\mbf{X}=\mbf{x})}{\delta(\mbf{x})} \diff\mbf{x}\diff m \nonumber\\
            &=\omega_{\sss\mbf{P}}\int_{-\infty}^{0}\frac{\int_{\R^d} P(\hat M=m,\mbf{X}=\mbf{x})/\delta(\mbf{x})\diff\mbf{x}}{\int_{\R^{d+1}} P(\hat M=m,\mbf{X}=\mbf{x})/\delta(\mbf{x})\diff\mbf{x}\diff m}\diff m
            =\omega_{\sss\mbf{P}} P(\hat M_\omega <0),
            \label{eq:w-prob-margin-neg}
        \end{align}
    where the last equality is given by a random variable $\hat{M}_{\omega}$ coming from the density $f_{\omega}$ defined as the expression inside the integral in Eq. \eqref{eq:w-prob-margin-neg}. 
    
    We further notice that the weighted first and second moments can be represented as:
    \begin{align*}
        \mu_1^{\hat M, \mbf{P}} &= \int_{\R^{d+1}} \frac{m}{\delta(\mbf{x})} P(\hat M=m,\mbf{X}=\mbf{x})\diff\mbf{x}\diff m= \omega_{\sss\mbf{P}} \mu^{\hat M_\omega}_1,\\
        \mu_2^{\hat M, \mbf{P}} &= \int_{\R^{d+1}} \frac{m^2}{\delta(\mbf{x})} P(\hat M=m,\mbf{X}=\mbf{x})\diff\mbf{x}\diff m= \omega_{\sss\mbf{P}} \mu^{\hat M_\omega}_2.
    \end{align*}
    From this, we also obtain that $var(M_\omega) = \left(\mu_2^{\hat M, \mbf{P}}/\omega_{\sss\mbf{P}}\right)-\left(\mu_1^{\hat M, \mbf{P}}/\omega_{\sss\mbf{P}}\right)^2$.
    Then, using the Cantelli-Chebyshev inequality (Lemma \ref{lem:cantelli-chebyshev}) with $\lambda=\mu^{\hat M_f}_1=\mu_1^{\hat M, \mbf{P}}/\omega_{\sss\mbf{P}}$ we deduce the following inequality:
    \begin{align}
    P(\hat M_\omega < 0) &\leq \frac{\left(\mu_2^{\hat M, \mbf{P}}/\omega_{\sss\mbf{P}}\right)-\left(\mu_1^{\hat M, \mbf{P}}/\omega_{\sss\mbf{P}}\right)^2}{\left(\mu_2^{\hat M, \mbf{P}}/\omega_{\sss\mbf{P}}\right)-\left(\mu_1^{\hat M, \mbf{P}}/\omega_{\sss\mbf{P}}\right)^2 + \left(\mu_1^{\hat M, \mbf{P}}/\omega_{\sss\mbf{P}}\right)^2} =1-\frac{\left(\mu_1^{\hat M, \mbf{P}}\right)^2}{\omega_{\sss\mbf{P}}\mu_2^{\hat M, \mbf{P}}}. 
    \label{eq:bound-for-w-prob-neg}
    \end{align}
    Combining Eq. \eqref{eq:bound-for-w-prob-neg} and Eq. \eqref{eq:expect-from-one-x-risk} we infer \eqref{eq:w-cbound}:
    \begin{align*}
        R(B_Q) &\leq \E_{\mbf{X}} \frac{\hat{r}(\mbf{x})}{\delta(\mbf{x})}-\E_{\mbf{X}}\frac{1-\alpha(\mbf{x})}{\delta(\mbf{x})} = \omega_{\sss\mbf{P}}P(\hat M_\omega < 0) - \omega_{\sss\mbf{P}} + \psi_{\mbf{P}}
        \leq \psi_{\mbf{P}} - \frac{\left(\mu_1^{\hat M, \mbf{P}}\right)^2}{\mu_2^{\hat M, \mbf{P}}}.
    \end{align*}
\end{proof}
Given data with imperfect labels, the direct evaluation of the generalization error rate may be biased, leading to an overly optimistic evaluation. Using the mislabeling matrix $\mathbf{P}$ we derive a more conservative C-bound, where the error of $\mbf{x}$ is penalized by the factor $1/\delta(\mbf{x})$. When there is no mislabeling, $\psi_{\mbf{P}}=1$, $\mu_1^{\hat M, \mbf{P}}$ and $\mu_2^{\hat M, \mbf{P}}$ are equivalent to $\mu_1^{\hat M}$ and $\mu_2^{\hat M}$, so we obtain the regular C-bound \eqref{eq:prob-cbound}.

In particular, this general result can be used to evaluate the error rate in the semi-supervised setting when mislabeling arises from pseudo-labeling of unlabeled examples via self-learning. Comparing with the transductive bound \eqref{eq:TB} obtained as a corollary of Theorem \ref{thm:tr-bound-bayes-multi}, \eqref{eq:w-cbound} directly upper bounds the error rate, so it will be tighter in most of cases. Particularly, it can be noticed that the value of \eqref{eq:TB} is  growing with the increase of the number of classes. Note that there exists other attempts to evaluate the C-bound in the semi-supervised setting. In the binary case, \cite{Lacasse:2007,Laviolette:2011} estimated the second margin moment using additionally unlabeled data by expressing it via disagreement of hypotheses. However, this holds for the binary case only.

\vasilii{In this theorem, we have combined the mislabeling bound \eqref{eq:one-x-mislabel-ineq} with the supervised multi-class C-bound \citep{Laviolette:2017}, however, another possibility could be to combine with the bound based on the second-order Markov's inequality \citep{Masegosa:2020}. As pointed out by \cite{Masegosa:2020}, the latter can be regarded as a relaxation of the C-bound, but it is easier to estimate from data in practice. Note that the tightest bound does not always imply the lowest error, so the use of C-bound in model selection tasks may be more advantageous as it involves both the individual strength of hypotheses and correlation between their errors, while the bound of \cite{Masegosa:2020} is based on the error correlation only.}

\subsection{PAC-Bayesian Theorem for C-Bound Estimation}
When the margin mean, the margin variance and the mislabeling matrix are empirically estimated from data, evaluation of \eqref{eq:w-cbound} may be optimistically biased. In this section, we analyze the behavior of the estimate with respect to the sample size. To achieve that, we use the PAC-Bayesian theory initiated by \citet{Mcallester:1999,McAllester:2003} to derive a Probably Approximately Correct bound defined below.

\begin{thm}
\label{thm:pac-bayesian-cbound}
Under the notations of Theorem \ref{thm:w-cbound}, for any set of classifiers $\mathcal{H}$, for any prior distribution $P$ on $\mathcal{H}$ and any $\epsilon \in (0,1]$, with a probability at least $1-\epsilon$ over the choice of the sample of size $n=l+u$, for every posterior distribution $Q$ over $\mathcal{H}$, if $\mu^{\hat{M}}_1>0$ and $\tilde{\delta}(\mbf{x})>0$, we have:
 \begin{align}
     R(B_Q) \leq \tilde\psi - \frac{ \tilde{\mu}_1^2}{\tilde{\mu}_2}, \label{eq:pac-bayes-bound}
 \end{align}
 where
\begin{align*}
    \tilde{\mu}_1 &= \frac{1}{u}\sum_{i=1}^u(1/\tilde{\delta}(\mbf{x}))\sum_{c=1}^K M_Q(\mbf{x}, c) P(Y\!=\!c|\mbf{X}\!=\!\mbf{x}) - B_1 \sqrt{\frac{2}{u}\left[\kld{Q}{P} + \ln\frac{2\sqrt{u}}{\epsilon}\right]}
    \\
    \tilde{\mu}_2 &= \frac{1}{u}(1/\tilde{\delta}(\mbf{x}))\sum_{i=1}^u\sum_{c=1}^K (M_Q(\mbf{x}_i, c))^2 P(Y\!=\!c|\mbf{X}\!=\!\mbf{x}_i) + B_2 \sqrt{\frac{2}{u}\left[2\kld{Q}{P} + \ln\frac{2\sqrt{u}}{\epsilon}\right]}
    \\
    \tilde{\psi} &=
    \frac{1}{u}\sum_{i=1}^u \frac{\tilde\alpha(\mbf{x}_i)}{\tilde\delta(\mbf{x}_i)} + B_3 \sqrt{\frac{2}{u} \ln\frac{2\sqrt{u}}{\epsilon}}\\
    \tilde{\delta}(\mbf{x}) &= \hat{\delta}(\mbf{x})-\sqrt{\frac{1}{2l_{c_\mbf{x}}}\ln\frac{2\sqrt{l_{c_\mbf{x}}}}{\epsilon}}-\sqrt{\frac{1}{2l_{j_\mbf{x}}}\ln\frac{2\sqrt{l_{j_\mbf{x}}}}{\epsilon}},\text{ with } c_\mbf{x}:=B_Q(\mbf{x}), j_\mbf{x}:=\argmin_{j\in\mathcal{Y}\setminus\{c_\mbf{x})\}}l_j,\\
    \tilde\alpha(\mbf{x}) &= \hat\alpha(\mbf{x}) + \sqrt{\frac{1}{2l_{c_\mbf{x}}}\ln\frac{2\sqrt{l_{c_\mbf{x}}}}{\epsilon}},
    \end{align*} 
    and where $\hat{\delta}(\mbf{x})$ and $\hat\alpha(\mbf{x})$ are empirical estimates respectively of $\delta(\mbf{x})$ and $\alpha(\mbf{x})$ based on the available labeled set,  $\kld{Q}{P}$ is the Kullback-Leibler divergence between $Q$ and $P$, and $l_j\!=\!\sum_{i = 1}^{l}\I{y_j=j}/l$  is the proportion of the labeled training examples from the true class $j$.
    \end{thm}
The proof is a combination of Propositions \ref{prop:pac-bayes-bound-first-moment}, \ref{prop:pac-bayes-bound-second-moment} and \ref{prop:pac-bound-psi} that are deferred to Appendix \ref{sec:appendix-cbound}.
    
Thus, by using Eq. \eqref{eq:pac-bayes-bound} we additionally penalize the C-bound by the sample size and the divergence between $Q$ and $P$. As $u$ grows, the penalization becomes less severe, so $\tilde{\mu}_1$ and $\tilde{\mu}_2$ are close to $\mu^{\hat{M}}_1$ and $\mu^{\hat{M}}_2$. Similarly, $\tilde{\delta}(\mbf x)$ and $\tilde{\alpha}(\mbf x)$ are closer to $\hat{\delta}(\mbf x)$ and $\hat{\alpha}(\mbf x)$ with the increase of the number examples used to estimate the mislabeling matrix, which we take $l$ for the sake of simplicity.
Note that, in contrast to the supervised case \citep[Theorem 3]{Laviolette:2017}, $B_1$ and $B_2$ can have a drastic influence on the bound's value, when $\tilde\delta(\mbf{x})$ is close to 0, which motivates in practice to use the $\lambda$-relaxation given by Eq. \eqref{eq:one-x-mislabel-ineq-with-lam}.

The obtained bound may be used to estimate the Bayes error from data, with the pseudo-labeled unlabeled examples serving as a hold-out set for estimating the margin moments, and the labeled examples serving as a hold-out set for estimating the mislabeling matrix. In the case of the random forest, the latter can be performed in the out-of-bag fashion as in \citep{Thiemann:2017,Lorenzen:2019}. However, the bound does not appear tighter in practice compared to the supervised case \citep{Laviolette:2017} due to the additional penalization on estimation of the mislabeling matrix. Making this bound tighter could be a good direction for future work. Nevertheless, when the focus is set on model selection, a common choice is to simply use an empirical estimate of the C-bound as an optimization criterion \citep{Bauvin:2020}.

\section{Experimental Results}
\label{sec:num-exper}

In this section, we describe numerical experiments that was performed to validate our proposed framework. At first, we test in practice the multi-class self-learning algorithm (denoted by \texttt{MSLA}) described in Section \ref{sec:msla} by comparing its ability to learn on partially labeled data with other classification algorithms. Then, we illustrate the proposed \eqref{eq:w-cbound} on real data sets and analyze its behavior.
All experiments were performed on a cluster with an \texttt{Intel(R) Xeon(R) CPU E5-2640 v3} at \texttt{2.60GHz}, \texttt{32} cores, \texttt{256GB} of RAM, the \texttt{Debian 4.9.110-3 x86\_64} OS.

\subsection{Experimental Setup}

Experiments are conducted on publicly available data sets \citep{Dua:2019,Chang:2011,Xiao:2017}. Since we are interested in the practical use of our approach in the semi-supervised context, we would like to see if it has good performance when $l\ll u$. Therefore, we do not use the train/test splits that are proposed by data sources. Instead, we propose our own splits that makes a situation closer to the semi-supervised context. Each experiment is conducted 20 times, by randomly splitting an original data set on a labeled and an unlabeled parts keeping fixed their respective size at each iteration. The reported performance results are averaged over the 20 trials. We evaluate the performance as the accuracy score over the unlabeled training set (\texttt{ACC-U}).

In all our experiments, we consider the Random Forest algorithm \citep{Breiman:2001} (denoted by \texttt{RF}) with 200 trees and the maximal depth of trees as the majority vote classifier with the uniform posterior distribution. For an observation $\mbf{x}$, we evaluate the vector of class votes $\{v(\mbf{x}, i)\}_{i=1}^K$ by averaging over the trees the vote given to each class by the tree. A tree computes a class vote as the fraction of training examples in a leaf belonging to a class.


Experiments are conducted on 11 real data sets. The associated applications are image classification with the \texttt{Fashion} data set, the \texttt{Pendigits} and the \texttt{MNIST} databases of handwritten digits; 
a signal processing application with the \texttt{SensIT} data set for vehicle type classification and the human activity recognition \texttt{HAR} database; 
speech recognition using the \texttt{Vowel}, the \texttt{Isolet} and the \texttt{Letter} data sets; 
document recognition using the \texttt{Page Blocks} database; and finally applications to bioinformatics with the \texttt{Protein} and \texttt{DNA} data sets. 
The main characteristics of these data sets are summarized in Table~\ref{tab:data set-description}.
\begin{table}[t]
  \centering
    \scalebox{0.82}{
      \begin{tabular}{ccccc}
        \toprule
        Data set & \# of labeled examples, & \# of unlabeled  examples, & Dimension, & \# of classes, \\
                 & $l$ & $u$ & $d$ & $K$ \\
        \midrule
        \texttt{Vowel} & 99 & 891 & 10 & 11 \\
        \texttt{Protein} & 129 & 951 & 77 & 8 \\
        \texttt{DNA} & 31 & 3155 & 180 & 3 \\
        \texttt{PageBlocks} & 1094 & 4379 & 10 & 5 \\
        \texttt{Isolet} & 389 & 7408 & 617 & 26 \\
        \texttt{HAR} & 102 & 10197 & 561 & 6 \\
        \texttt{Pendigits} & 109 & 10883 & 16 & 10 \\
        \texttt{Letter} & 400 & 19600 & 16 & 26 \\
        \texttt{Fashion} & 175 & 69825 & 784 & 10 \\
        \texttt{MNIST} & 175 & 69825 & 784 & 10 \\
        \texttt{SensIT} & 49 & 98479 & 100 & 3 \\
        \bottomrule
      \end{tabular}}
\caption{Characteristics of data sets used in our experiments ordered by the size of the 
training set $(n=l+u)$.}
\label{tab:data set-description}
\end{table}


The proposed \texttt{MSLA} that automatically finds the threshold by minimizing the conditional Bayes error rate, is compared with the following baselines: 
\begin{itemize}
    \item a fully supervised \texttt{RF} trained using only labeled examples. The approach is obtained at the initialization step of \texttt{MSLA} and once learned it is directly applied to predict the class labels of the whole unlabeled set; 
    \item the scikit-learn implementation \citep{scikit-learn} of the graph based, label spreading algorithm \citep{Zhou:2004} denoted by \texttt{LS};
    \item the one-versus-all extension of a transductive support vector machine \cite{Joachims:1999} using the Quasi-Newton scheme. The approach was proposed by \citet{Gieseke:2014} ans is further denoted as \texttt{QN-S3VM}\footnote{The source code for the binary \texttt{QN-S3VM} is available at \url{http://www.fabiangieseke.de/index.php/code/qns3vm}.};
    \item a semi-supervised extension of the linear discriminant analysis \texttt{Semi-LDA}, which is based on the contrastive pessimistic likelihood estimation proposed by \cite{Loog:2015};
    \item a semi-supervised extension of the random forest \texttt{DAS-RF} proposed by \cite{Leistner:2009} where the classifier is repeatedly re-trained on the labeled and all the unlabeled examples with pseudo-labels optimized via deterministic annealing;
    \item the multi-class extension of the classical self-learning approach (denoted by \texttt{FSLA}) described in \citet{Tur:2005} with a fixed prediction vote threshold;
    \item a self-learning approach (denoted by \texttt{CSLA}) where the threshold is defined via curriculum learning by taking it as the $(1-t\cdot\Delta)$-th percentile of the prediction vote distribution at the step $t=1,2,\dots$ \citep{Cascantebonilla:2020}.
\end{itemize}

As the size of the labeled training examples $|\mathrm{Z}_{\mathcal{L}}|$ is small, the hyperparameter tuning can not be performed properly. At the same time, the performance of baselines may be sensitive to some of their hyperparameters.  For this reason, we compute \texttt{LS}, \texttt{QN-S3VM}, \texttt{Semi-LDA}, \texttt{DAS-RF} on a grid of parameters' values, and then choose a hyperparameter for which the performance is the best in average on 20 trials. We tune the RBF kernel parameter $\sigma\in\{10, 1.5, 0.5, 10^{-1}, 10^{-2}, 10^{-3}\}$ for \texttt{LS}, the regularization parameters $(\lambda,\lambda')\in\{10^{-1}, 10^{-2}, 10^{-3}\}^2$ for \texttt{QN-S3VM}, the learning rate $\alpha\in\{10^{-4},10^{-3},10^{-2}\}$ for \texttt{Semi-LDA}, the initial temperature $T_0\in\{10^{-3}, 5\cdot10^{-3}, 10^{-2}\}$ for \texttt{DAS-RF}. Other hyperparameters for these algorithms are left to their default values. Particularly, in \texttt{DAS-RF} the strength parameter and the number of iterations are respectively set to 0.1 and 10.
 
While the aforementioned parameters are rather data-dependent, the choice of $\theta$ for \texttt{FSLA} and $\Delta$ for \texttt{CSLA} depend more on what prediction vote distribution the base classifier outputs. After manually testing different values, we have found that \texttt{FSLA}$_{\theta=0.7}$ and \texttt{CSLA}${_\Delta=1/3}$ are good choices for the random forest. For \texttt{FSLA}, we terminate the learning procedure as soon as the algorithm makes 10 iterations, which reduces the computation time and may also improve the performance, since, in this case, the algorithm is less affected by noise. \cite{Cascantebonilla:2020} used for \texttt{CSLA} a slightly other architecture for self-learning, where the set of selected pseudo-labeled examples included just for one iteration (like if in Algorithm 1 Step 3 would be replaced by $\mathrm{Z}_{\mathcal{P}}\leftarrow S$).  In our context, we have found that the performance of \texttt{CSLA} is identical for both two architectures.

\subsection{Illustration of MSLA}
\label{sec:msla-exp}

In our setup, a time deadline is set: we stop computation for an algorithm if one trial takes more than 4 hours. Table \ref{tab:multi-class-exp-res} summarizes results obtained by \texttt{RF}, \texttt{LP}, \texttt{QN-S3VM}, \texttt{Semi-LDA}, \texttt{DAS-RF}, \texttt{FSLA}, \texttt{CSLA} and \texttt{MSLA}. We used bold face to indicate the highest performance rates and the  symbol $\downarrow$ indicates that the performance is significantly worse than the best result, according to Mann-Whitney U test \citep{Mann:1947} used at the p-value threshold of 0.01.

\setlength{\tabcolsep}{0.45em}
\begin{table}[t]
  \centering
  {\scalebox{0.7}{
      \begin{tabular}{l|cccccccc}
        \toprule
        Data set & \texttt{RF} &  \texttt{LS} & \texttt{QN-S3VM} & \texttt{Semi-LDA} & \texttt{DAS-RF} & \texttt{FSLA$_{\,\bm{\theta} = 0.7}$} & \texttt{CSLA$_{\,\Delta = 1/3}$} & \texttt{MSLA} \\
        \midrule
         \texttt{Vowel} & $.586 \pm .028$ & $\textbf{.602} \pm .026$ & $.208^\downarrow \pm .029$ & .432$^\downarrow$ $\pm$ .029 & .587 $\pm$ .028 & .531$^\downarrow$ $\pm$ .034 & .576$^\downarrow$ $\pm$ .031 & .586 $\pm$ .026 \\
         \midrule
        \texttt{Protein} & $.764^\downarrow \pm .032$ & $.825 \pm .028$ & $.72^\downarrow \pm .034$ & \textbf{.842} $\pm$ .029 & .768$^\downarrow$ $\pm$ .036 & .687$^\downarrow$ $\pm$ .036 & .771$^\downarrow$ $\pm$ .035 & .781$^\downarrow$ $\pm$ .034 \\
        \midrule
        \texttt{DNA} & $.693^\downarrow \pm .074$ & $.584^\downarrow \pm .038$ & $\textbf{.815} \pm .025$ & .573$^\downarrow$ $\pm$ .037 & .693$^\downarrow$ $\pm$ .083 & .521$^\downarrow$ $\pm$ .095 & .671$^\downarrow$ $\pm$ .112 & .702$^\downarrow$ $\pm$ .082 \\
        \midrule
        \texttt{PageBlocks} & $.965 \pm .003$ & $.905^\downarrow \pm .004$ & $.931^\downarrow \pm .003$ & .935$^\downarrow$ $\pm$ .009 & .965 $\pm$ .003 & .964 $\pm$ .004 & .965 $\pm$ .003 & \textbf{.966} $\pm$ .002 \\
        \midrule
        \texttt{Isolet} & $.854^\downarrow \pm .016$ & $.727^\downarrow \pm .01$ & $.652^\downarrow \pm .016$ & .787$^\downarrow$ $\pm$ .019 & .859$^\downarrow$ $\pm$ .018 & .7$^\downarrow$ $\pm$ .04 & .843$^\downarrow$ $\pm$ .021 & \textbf{.875} $\pm$ .014 \\
        \midrule
        \texttt{HAR} & $.851 \pm .024$ & $.215^\downarrow \pm .05$ & $.78^\downarrow \pm .02$ & .743$^\downarrow$ $\pm$ .043 & .852 $\pm$ .024 & .81$^\downarrow$ $\pm$ .041  & .841 $\pm$ .029 & \textbf{.854} $\pm$ .026 \\
        \midrule
        \texttt{Pendigits} & $.863^\downarrow \pm .022$ & $\textbf{.916} \pm .013$ & $.675^\downarrow \pm .022$ & .824$^\downarrow$ $\pm$ .012 & .872$^\downarrow$ $\pm$ .023 & .839$^\downarrow$ $\pm$ .036 & .871$^\downarrow$ $\pm$ .029 & .884$^\downarrow$ $\pm$ .022 \\
        \midrule
        \texttt{Letter} & $.711 \pm .011$ & $.664^\downarrow \pm .01$ & $.064^\downarrow \pm .013$ & .589$^\downarrow$ $\pm$ .016 & .718 $\pm$ .012 & .651$^\downarrow$ $\pm$ .015 & \textbf{.72} $\pm$ .013  & .717 $\pm$ .013 \\
        \midrule
        \texttt{Fashion} & $.718 \pm .022$ & \texttt{NA} & \texttt{NA} & .537$^\downarrow$ $\pm$ .027 & .722 $\pm$ .023 & .64$^\downarrow$ $\pm$ .04 & .713 $\pm$ .026 & \textbf{.723} $\pm$ .023 \\
        \midrule
        \texttt{MNIST} &  $.798^\downarrow \pm .015$  &  \texttt{NA} & \texttt{NA} & .423$^\downarrow$ $\pm$ .029 & .822$^\downarrow$ $\pm$ .017 & .705$^\downarrow$ $\pm$ .055 & .829$^\downarrow$ $\pm$ .02  & \textbf{.857} $\pm$ .013 \\
        \midrule
        \texttt{SensIT} &  $\textbf{.723} \pm .022$ & \texttt{NA} & \texttt{NA} &  .647$^\downarrow$ $\pm$ .042 & \textbf{.723} $\pm$ .022 & .692$^\downarrow$ $\pm$ .023 & .713 $\pm$ .024 & .722 $\pm$ .021 \\
        \bottomrule
      \end{tabular}}}
\caption{Classification performance on different data sets described in Table \ref{tab:data set-description}. The performance is computed using the accuracy score on the unlabeled training examples (\texttt{ACC-U}). The sign $^\downarrow$ shows if the performance is statistically worse than the best result on the level 0.01 of significance. \texttt{NA} indicates the case when the time limit was exceeded.}
\label{tab:multi-class-exp-res}
\end{table}

From these results it comes out that

 \begin{itemize}
    \item in 5 of 11 cases, the \texttt{MSLA} performs better than its opponents. On data sets \texttt{Isolet} and \texttt{MNIST} it significantly outperforms all the others, and it significantly outperforms the baseline \texttt{RF} on \texttt{Isolet}, \texttt{Pendigits} and \texttt{MNIST}\,(6\% improvement);
    \item the \texttt{LS} and the \texttt{QN-S3VM} did not pass the scale over larger data sets (\texttt{Fashion}, \texttt{MNIST} and \texttt{SensIT}), while the \texttt{MSLA} did not exceeded 2 minutes per trial on these data sets (see Table \ref{tab:computationTime});
    \item the performance of \texttt{LS} and \texttt{Semi-LDA} performance varies greatly on different data sets, which may be caused by the topology of data. In contrast, \texttt{MSLA} has more stable results over all data sets as it is based on the predictive score, and the \texttt{RF} is used as the base classifier;
     \item since the \texttt{QN-S3VM} is a binary classifier by nature, its one-versus-all extension is not robust with respect to the number of classes. This can be observed on \texttt{Vowel}, \texttt{Isolet} and \texttt{Letter}, where the number of classes is high;
     \item from our observation, both \texttt{LS} and \texttt{QN-S3VM} are highly sensitive to the choice of the hyperparameters. However, it is not very clear whether these hyperparameters can be properly tuned given a insufficient number of labeled examples. The same concern is applied to all the other semi-supervised baselines, while \texttt{MSLA} does not require any particular tuning since it finds automatically the threshold $\bm{\theta}$;
    \item while the approach proposed by \cite{Loog:2015} always guarantees an improvement of the likelihood compared to the supervised case, we have observed that the classification accuracy is not always improved for \texttt{Semi-LDA} and may even degrade over the supervised linear discriminant analysis;
    
    \item compared to the fully supervised approach, \texttt{RF}, the use of pseudo-labeled unlabeled training data (in \texttt{DAS-RF}, \texttt{FSLA}, \texttt{CSLA} or \texttt{MSLA}) may generally give no benefit or even degrade performance in some cases (\texttt{Vowel}, \texttt{PageBlocks}, \texttt{SensIT}). This may be due to the fact that the learning hypotheses are not met regarding the data sets where this effect is observed;
    
    \item although for \texttt{DAS-RF} the performance is usually not degraded when $T_0$ is properly chosen, it has rather little improvement compared to \texttt{RF}. The performance of \texttt{FSLA} degrades most of the time, while degradation for \texttt{CSLA} is observed on 6 data sets. The latter suggests that the choice of the threshold for pseudo-labeling is crucial and challenging in the multi-class framework. Using the proposed criterion based on Eq. \eqref{eq:cond-bayes-error}, we can find the threshold efficiently;
    
    \item from the results it can be seen that self-learning is also sensitive to the choice of the initial classifier. On some data sets, the number of labeled examples might be too small leading to a bad initialization of the first classifier trained over the labeled set. This implies that the initial votes are biased, so even with a well picked threshold we do not expect a great increase in performance (see Appendix \ref{sec:posterior-estimation} for more details).
 \end{itemize}

\begin{figure}[ht!]
\centering
\includegraphics[width=0.5 \textwidth]{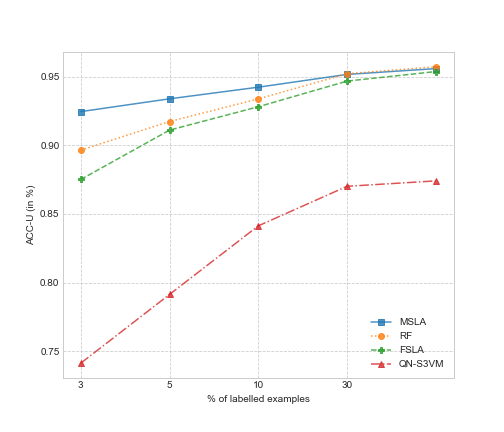}
\caption{Classification accuracy with respect to the proportion of unlabeled examples for the \texttt{MNIST} data set (a subsample of 3500 examples). On the graph, dots represent the average performance on the unlabeled examples over 20 random splits. For simplicity of illustration, the other considered algorithms are not displayed.} 
\label{fig:SmallMNIST}
\end{figure}

We also analyze the behavior of the various algorithms for growing initial amounts of labeled data in the training
set. Figure \ref{fig:SmallMNIST} illustrates this by showing the accuracy on a subsample of 3500 observations from \texttt{MNIST} of \texttt{RF}, \texttt{QN-S3VM}, \texttt{FSLA}$_{\bm{\theta}=0.7}$ and \texttt{MSLA} with respect to the percentage of the labeled training examples. In this graph, the performance of \texttt{LS} is not depicted, since it is significantly lower compared to the other methods under consideration. As expected, all performance curves increase monotonically with respect to the additional labeled data. When there are sufficient labeled training examples, \texttt{MSLA}, \texttt{FSLA} and \texttt{RF} actually converge to the same accuracy performance, suggesting that the labeled data carries out sufficient information and no additional information could be extracted from unlabeled examples.


Further, we present a comparison of the learning algorithms under consideration by analyzing their complexity.
The time complexity of the random forest \texttt{RF} is $O(T d \tilde{l}\log^2 \tilde{l})$ \citep{Louppe:2014}, where $T$ is the number of decision trees in the forest and $\tilde{l}\approx 0.632\cdot l$ is the number of training examples used for each tree. Since \texttt{RF} is employed in \texttt{DAS-RF} and self-learning, the time complexity of \texttt{DAS-RF}, \texttt{FSLA} and \texttt{CSLA}  is $O(C T d\tilde{n}\log^2 \tilde{n})$, where $C$ is the number of times \texttt{RF} has been learned, $\tilde{n}\approx 0.632\cdot n$. In our experimental setup, $C=11$ for \texttt{FSLA} and \texttt{DAS-RF}, and $C=1/\Delta +1 = 4$ for \texttt{CSLA}.

The time required for finding the optimal threshold at every iteration of the \texttt{MSLA} is $O(K^2 R^2 n)$, where $R$ is the sampling rate of the grid. From this we deduce that the complexity of \texttt{MSLA} is $O(C\max(T d n\log^2 n, K^2 R^2 n))$. As $n$ grows, the complexity is written as $O(d n\log^2 n)$, since $C, T, R$ are constant. This indicates a good scalability of all considered pseudo-labeling methods for large-scale data as they also have a memory consumption proportional to $nd$, so the computation can be performed on a regular PC even for the large-scale applications.

In the label spreading algorithm, an iterative procedure is performed, where at every step the affinity matrix is computed. Hence, the time complexity of the \texttt{LS} is $O(M n^2 d)$, where $M$ is the maximal number of iterations. From our observation, the convergence of \texttt{LS} is highly influenced by the value of $\sigma$ and the data topology. The time complexity of the \texttt{QN-S3VM} is $O(n^2 d)$ \citep{Gieseke:2014}. Both algorithms suffer from high run-time for large-scale applications. Since \texttt{LS} and \texttt{QN-S3VM} evaluate respectively the affinity matrix and the kernel matrix of size $n$ by $n$, these algorithms have also large space complexity proportional to $n^2$. From our observation, for the large-scale data (\texttt{Fashion}, \texttt{MNIST}, \texttt{SensIT}) the maximal resident set size\footnote{Maximal resident set size (maxRSS) is the peak portion of memory that was occupied in RAM during the run.} of \texttt{LS} and \texttt{QN-S3VM} may reach up to 200GB of RAM, which is practically infeasible with lack of resources.

Finally, the time complexity of \texttt{Semi-LDA} is $O(M\max(nd^2, d^3))$, where $M$ is the maximal number of iterations and $O(\max(nd^2, d^3))$ is the complexity of the linear discriminant analysis assuming $n>d$ \citep{Cai:2008}, and the space complexity is $O(nd)$. The approach pass the scale well with respect to the sample size, but may significantly slow down in the case of very large dimension. In Section \ref{sec:run-time}, we further analyze the time complexity empirically for all the methods under consideration.

\subsection{Illustration of (CBIL)}
\label{sec:cbil-exp}
In this section, we illustrate the value of \eqref{eq:w-cbound} evaluated on the unlabeled examples pseudo-labeled by \texttt{MSLA}. We study how the bound's value is penalized by the mislabeling model, so we empirically compare it with the oracle C-bound \eqref{eq:prob-cbound} evaluated as if the labels for the considered  unlabeled data would be known. 

To do so, we compute the value of the two bounds varying the number of  examples used for evaluation with respect to the prediction confidence: the pseudo-labeled examples are sorted by the value of the prediction vote in the descending order, and we keep only the first $\rho\%$ of the examples for $\rho \in \{20, 40, 60, 80, 100\}$.

\begin{figure}[h!]
    \includegraphics[width=\textwidth, trim =  0.8cm 0.8cm 1cm 1.3cm, clip=TRUE]{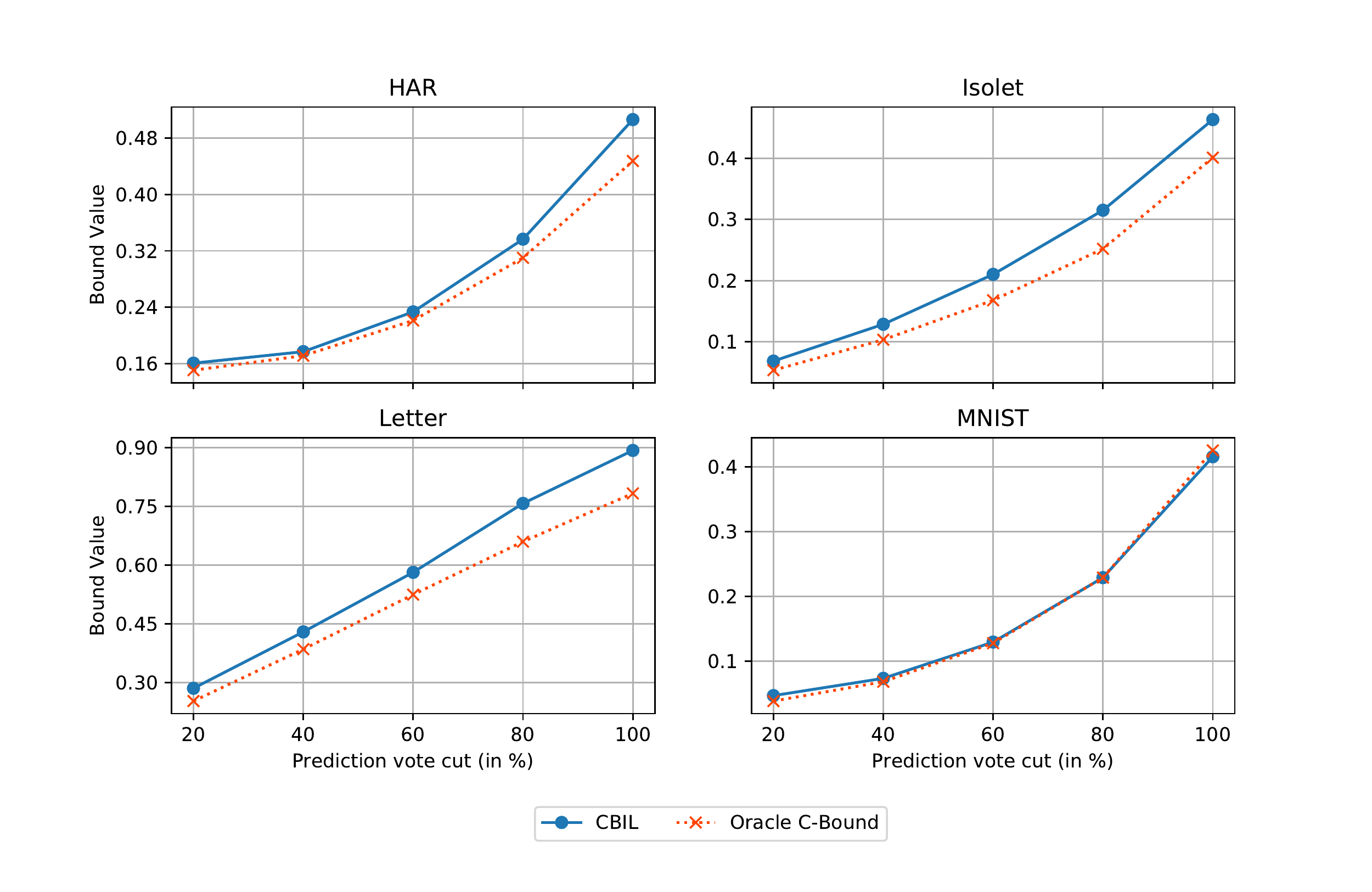}
    \caption{\eqref{eq:w-cbound} and Oracle C-Bound when varying the number of pseudo-labels on 4 data sets. We keep the most confident one (with respect to prediction vote) from $20\%$ to $100\%$.}
    \label{fig:2}
\end{figure}

We use the votes of the current classifier and expect that with increase of $\rho$ we have more mislabels, so the \eqref{eq:w-cbound} is more penalized. In \eqref{eq:w-cbound}, we use the true value of the mislabeling matrix (i.e., evaluated using the labels of unlabeled data) for clear illustration of the C-bound's penalization. In Section \ref{sec:concl}, we discuss the possible estimations of the mislabeling matrix.
 
The experimental results on 4 data sets \texttt{HAR}, \texttt{Isolet}, \texttt{Letter} and \texttt{MNIST} are illustrated in Figure \ref{fig:2}. 
As expected, the classifier makes mistakes mostly on low class votes, so the error increases when $\rho$ grows. One can see that on \texttt{Isolet}, \texttt{HAR} and \texttt{Letter}  \eqref{eq:w-cbound} is close to the oracle C-bound for small $\rho$, since most of pseudo-labels are true. When more noisy pseudo-labels are included, the difference between the two values becomes more evident, leading \eqref{eq:w-cbound} to be more pessimistic. This is probably connected with the choice of the mislabeling error model \eqref{eq:mislab-model} that is class-related and not instance-related. Although we lose some flexibility, the class-related mislabeling matrix would be easier to estimate in practice.
Finally, for \texttt{MNIST}, the two bounds are very close to each other, and the mislabeling is occasional, which is agreed with Table \ref{tab:multi-class-exp-res} as pseudo-labels are very helpful on this data set.

\section{Conclusion and Future Work}
\label{sec:concl}
In this paper, we proposed a new probabilistic framework for the multi-class semi-supervised learning. At first, we derived a bound for the transductive conditional risk of the majority vote classifier. This probabilistic bound is based on the distribution of the class vote over unlabeled examples for a predicted class. We deduced corresponding bounds on the confusion matrix norm and the error rate as a corollary and determined when the bounds are tight. 
Then, we proposed a multi-class self-learning algorithm where the threshold for selecting unlabeled data to pseudo-label is automatically found from minimization of the transductive bound on the majority vote error rate.
From the numerical results, it came out that the self-learning algorithm is sensitive to the supervised performance of the base classifier, but it can better pass the scale and significantly outperform the case when the threshold is manually fixed. 
%
However, the pseudo-labels produced by self-learning are imperfect, so we proposed a mislabeling error model to take explicitly into account these mislabeling errors.
We established the connection between the true and the imperfect output and consequently extended the  C-bound to imperfect labels, and derived a PAC-Bayesian Theorem for controlling the sample effect. The proposed bound allowed us to evaluate the performance of the learning model after pseudo-labeling the unlabeled data.
We illustrated the influence of the mislabeling error model on the bound's value on several real data sets.

We raise several open practical questions, which we detail below and leave as a subject for future work.\\
Firstly, the proposed self-learning policy has been experimentally validated when it is coupled with the random forest, but it would be interesting to test also with deep learning methods. This, however, is not straightforward. It is well known that the modern neural networks are not well calibrated, and examples are often misclassified with a high prediction vote \citep{Guo:2017}. This is a significant limitation in our case, since we make an assumption that the classifier makes its mistakes on examples with low prediction votes, which is used for the bound's approximation. Possible solutions include the use of neural network ensembles or temperature scaling.  
%
\\
Secondly, further analysis of the learning model learned on pseudo-labels is perplexing due to the so-called \textit{confirmation bias}: at every iteration, the self-learning includes into the training set unlabeled examples with highly confident predictions, which arise from classifier's overconfidence to its initial decisions that could be erroneous. This implies that the hypotheses will have small disagreement on the unlabeled set after pseudo-labeling, so the votes are no more adequate for measuring prediction confidence. A correct estimation of mislabeling probabilities or changing the way self-learning is learned are possible solutions.\\
%
Thirdly, \eqref{eq:w-cbound} requires in practice the estimation of the mislabeling matrix, which is a complex problem, but an active field of study \citep{Natarajan:2013}. Most of these studies tackle this problem from an algorithmic point of view: 
for example, in the semi-supervised setting, \cite{Krithara:2008} learn the mislabeling matrix together with the classifier
parameters through the classifier likelihood maximization for document classification;
in the supervised setting, a common approach is to detect anchor points whose labels are surely true \citep{Scott:2015}. 
A potential idea would be to transfer this idea to the semi-supervised case in order to detect the anchor points in the unlabeled set and use them together with the labeled set for correct estimation of the noise in pseudo-labels; this may require additional assumptions such as the existence of clusters \citep{Rigollet:2007,Maximov:2018} or manifold structure \citep{Belkin:2004}.
We also point out possible applications of \eqref{eq:w-cbound}.
At first, the bound can be used for model selection tasks as  semi-supervised feature selection \citep{Sheikhpour:2017}. 
Since minimization of the C-bound implies simultaneously maximization of the margin mean  and minimization of the margin variance, \eqref{eq:w-cbound} would guide a feature selection algorithm to choose an optimal feature subset based on the labeled and the pseudo-labeled sets. 
\\
Next, \eqref{eq:w-cbound} can be used as a criterion to learn the posterior $Q$
in the semi-supervised setting. This issue is actively studied in the supervised context, e.g., \cite{Roy:2016,Bauvin:2020} have been developed the boosting-based C-bound optimization algorithms.
\\
It should be noticed that for these two applications, the main objective is to rank models, so the best model has the minimal error on the unlabeled set. Hence, the bound analysis goes beyond the classical question of tightness: the tightest bound does not always imply the minimal error, and a bound relaxation can have a positive effect (see Appendix \ref{sec:relax_bound}).

%





\bibliographystyle{apalike}
\bibliography{bibjmlr.bib}

\newpage
\appendix
\section{Tools for Section \ref{sec:tr-study}}
\subsection{Tools for Theorem 3.2}
\label{AppendixProofLemma}
\begin{proof}[Proof of Lemma \ref{lem:connection-Gibbs-Bayes-multi}]
First, we obtain Eq. \eqref{eq:lemma:gibbs:multi}:
\begin{align*}
R_\mathcal{U}(G_Q,i,j) &= \frac{1}{u_i} \E_{h\sim Q}\sum_{\mathbf{x}\in X_\mathcal{U}} P(Y=i|X=\mathbf{x})\I{h(\mathbf{x}) = j} = \frac{1}{u_i} \sum_{\mathbf{x}\in X_\mathcal{U}} P(Y=i|X=\mathbf{x})v_Q(\mathbf{x},j) \\
&\geq \frac{1}{u_i} \sum_{\mathbf{x}\in X_\mathcal{U}} P(Y=i|X=\mathbf{x})v_Q(\mathbf{x},j)\I{B_Q(\mathbf{x})=j}\\
&= \frac{1}{u_i} \sum_{t=1}^{N_j}\sum_{\mathbf{x}\in X_\mathcal{U}} \left(P(Y=i|X=\mathbf{x})\I{B_Q(\mathbf{x})=j}\I{v_Q(\mathbf{x},j)=\gamma^{(t)}_j}\right)\gamma^{(t)}_j= \sum_{t=1}^{N_j} b_{i,j}^{(t)}\gamma^{(t)}_j.
\end{align*}
Then, we deduce Eq. \eqref{eq:lemma:bayes:multi}:
\begin{align*}
    R_\mathcal{U\wedge\bm{\theta}}(B_Q,i,j) &= \frac{1}{u_i} \sum_{\mathbf{x}\in X_\mathcal{U}} P(Y=i|X=\mathbf{x})\I{B_Q(\mathbf{x}) = j}\I{v_Q(\mathbf{x},j)\geq \theta_j} \\
    &= \frac{1}{u_i} \sum_{t=1}^{N_j}\sum_{\mathbf{x}\in X_\mathcal{U}} P(Y=i|X=\mathbf{x})\I{B_Q(\mathbf{x}) = j}\I{v_Q(\mathbf{x},j)=\gamma^{(t)}_j}\I{\gamma^{(t)}_j\geq \theta_j} \\
    &= \frac{1}{u_i} \sum_{t=k_j+1}^{N_j}\sum_{\mathbf{x}\in X_\mathcal{U}} P(Y=i|X=\mathbf{x})\I{B_Q(\mathbf{x}) = j}\I{v_Q(\mathbf{x},j)=\gamma^{(t)}_j} = \sum_{t=k_j+1}^{N_j} b_{i,j}^{(t)}.
\end{align*}
\end{proof}

\begin{lem}[Lemma 4 in \citet{Amini:2008}]
\label{lem:sol-lin-prog}
Let $(g_i)_{i \in \{ 1,\ldots,N\}}$ be such that $0<g_1<\dots<g_N\leq 1$. Consider also $p_i\geq 0$ for each $i\in\{1,\dots,N\}$, $B\geq 0$, $k\in\{1,\dots,N\}$. Then, the optimal solution of the linear program:
\[
\begin{cases}
\max_{\mbf{q}:=(q_1,\dots,q_N)} F(\mbf{q}) := \max_{q_1,\dots,q_N}\sum_{i=k+1}^N q_i\\
0\leq q_i\leq p_i\quad \forall i \in \{ 1,\ldots,N\}\\
\sum_{i=1}^N q_i g_i\leq B
\end{cases}
\]
will be $\mbf{q}^*$ defined as, for all $i\in \{ 1,\ldots,N\}$, $q^*_i=\min\left(p_i, \floor*{\frac{B-\sum_{j<i}q^*_jg_j}{g_i}}_+\right)\I{i>k}$; where, the sign $\floor{\cdot}_+$ denotes the positive part of a number, $\floor*{x}_+ = x\cdot \I{x>0}$.

\end{lem}

\begin{proof}[Proof of Lemma A.1]
It can be seen that the first $k$ target variables should be zero for the optimal solution. Indeed, they do not influence explicitly the target function $F$. However, terms $g_iq_i$ for $i \in \{1,\ldots, k\}$ are positive, so their increase leads to smaller values of $q_i$ for $i \in \{k+1,\ldots, N\}$, which in their turn decrease the value of $F$. Because of this, we look for a solution in a space $\mathcal{O}=\{0\}^k\times  \prod_{i=k+1}^N [0,p_i]$. We aim to show that there is a unique optimal solution $\mbf{q}^*$ in $\mathcal{O}$.

\textbf{Existence.}
It is known that the linear program under consideration is a convex, feasible and bounded task. Hence, there is a feasible optimal solution $\mbf{q}^{opt}\in\prod_{i=1}^N[0, p_i]$. Then, we define  $\mbf{q}^{opt,\mathcal{O}}\in\mathcal{O}$:
\[
\begin{cases}
q_i^{opt,\mathcal{O}} = q_i^{opt} & \text{if } i>k\\
q_i^{opt,\mathcal{O}} = 0 & \text{otherwise}.
\end{cases}
\]
It can be seen that this solution is feasible: $F(\mbf{q}^{opt,\mathcal{O}}) = F(\mbf{q}^{opt})$. Then, there exists an optimal solution in $\mathcal{O}$. Further, the optimal solution is again designated as $\mbf{q}^*$.

\textbf{Unique representation.}
We would like to find a representation of $\mbf{q}^*$ that is, in fact, unique. Before doing it, one can notice that for $\mbf{q}^*$ the following equation is necessarily true:
\[
 \sum_{i=1}^N q_i^*g_i = B.
\]
Indeed, as $g_i$ are fixed, $\mbf{q}^*$ would not be optimal otherwise, and there would exist $\tilde{\mbf{q}}$ such that $\sum_{i=1}^N \tilde{q}_ig_i > \sum_{i=1}^N q_i^*g_i $, which implies $F(\tilde{\mbf{q}})>F(\mbf{q}^*)$. 

Let's consider the lexicographic order $\succeq$:
{\small{
\begin{multline*}
\forall(\mbf{q},\mbf{q}')\in\R^N\times\R^N, \mbf{q}\succeq \mbf{q}' \Leftrightarrow 
\left\{\mathcal{I}(\mbf{q}',\mbf{q}) = \emptyset\right\}\ \vee \left\{\mathcal{I}(\mbf{q}',\mbf{q}) \not= \emptyset \wedge \min\left(\mathcal{I}(\mbf{q},\mbf{q}')\right)<\min\left(\mathcal{I}(\mbf{q}',\mbf{q})\right)\right\},
\end{multline*}
}}
where $\mathcal{I}(\mbf{q}',\mbf{q}) = \{i|q'_i>q_i\}$.

We aim to show that the optimal solution is actually the greatest feasible solution in $\mathcal{O}$ \\for $\succeq$. Let $\mathcal{M}$ be the set$\{i>k|q^*_i<p_i\}$. Then, there are two cases:
\begin{itemize}
    \item $M=\emptyset$. It means that for all $i>k$, $q^*_i=p_i$ and $\mbf{q}^*$ is then the maximal element for $\succeq$ in $\mathcal{O}$.
    \item $M\not=\emptyset$. Let's consider $K=\min\{i>k|q^*_i<p_i\},\ M = \mathcal{I}(\mbf{q},\mbf{q}^*)$. By contradiction, suppose $\mbf{q}^*$ is not the greatest feasible solution for $\succeq$ and there is $\mbf{q}\in\R^N$ such that $\mbf{q}\succ \mbf{q}^*$.
        \begin{enumerate}
            \item $M\leq k$. Then, $q_M > q^*_M = 0$. It implies that $\mbf{q}\not\in\mathcal{O}$.
            \item $k<M<K$. Then, $q_M > q^*_M = p_M$. The same, $\mbf{q}\not\in\mathcal{O}$.
            \item $M\geq K$. Then, $F(\mbf{q})>F(\mbf{q}^*)$. But it means that $\sum_{i=1}^N q_ig_i > \sum_{i=1}^N q_i^*g_i  = B$.  
        \end{enumerate}
\end{itemize}

Hence, we conclude that if the solution is optimal then it is necessarily the greatest feasible solution for $\succeq$. Let's prove that if a solution is not the greatest feasible one then it can not be optimal. With this statement, uniqueness would be proven.

Consider $\mbf{q}\in\mathcal{O}$ such that $\mbf{q}^*\succ \mbf{q}$.
\begin{itemize}
    \item $\mathcal{I}(\mbf{q},\mbf{q}^*) = \emptyset$. Then, $F(\mbf{q}^*)>F(\mbf{q})$ and $\mbf{q}$ is not optimal.
    \item $\mathcal{I}(\mbf{q},\mbf{q}^*) \not= \emptyset$. Let $K=\min\left(\mathcal{I}(\mbf{q}^*,\mbf{q})\right)$ and $M = \min\left(\mathcal{I}(\mbf{q},\mbf{q}^*)\right)$. Then, $q_M>q^*_M\geq 0$ and $K<M$. 
    Denote $\lambda=\min\left(q_M, \frac{g_M}{g_K}(p_K-q_K)\right)$ and define $\mbf{q}'$ by:
    \[
    q'_i = q_i,\ i\not\in\{K,M\}, \quad q'_K = q_K + \frac{g_M}{g_K}\lambda \quad q'_M = q_M-\lambda
    \]
\end{itemize}
It can be observed that $\mbf{q}'$ satisfies the box constraints. Moreover, $F(\mbf{q}') = F(\mbf{q})+\lambda(g_M/g_K - 1)>F(\mbf{q})$ since $g_K<g_M$ and $\lambda>0$. Thus, $\mbf{q}$ is not optimal.
Summing up, it is proven that there is the only optimal solution in $\mathcal{O}$ and it is the greatest feasible one for $\succeq$.

Then, let's obtain an explicit representation of this solution. As it is the greatest one in lexicographical order, we assign $q_i$ for $i>k$ to maximal feasible values, which are $p_i$. It continues until the moment when $\sum_{j=1}^i q_ig_i$ is close to $B$. Denote by $I$ the index such that $\sum_{i=1}^{I-1} p_ig_i\leq B$, but $\sum_{i=1}^{I} p_ig_i\geq B$.
\begin{itemize}
    \item $\sum_{i=1}^{I-1} p_ig_i=B$. Then, $q_i = 0$ for $i\geq I$. It can be also written in the following way: 
    $$q_i = \floor*{\frac{B-\sum_{j<i}q_jg_j}{g_i}}_+, \qquad i\geq I$$.
    \item $\sum_{i=1}^{I-1} p_ig_i<B$. Then, $q_I$ is equal to residual:
    $$q_I = \frac{B-\sum_{j<I}q_jg_j}{g_I} = \floor*{\frac{B-\sum_{j<I}q_jg_j}{g_I}}_+.$$
    For the other $q_i$, $i>I$ we assign to 0.
\end{itemize}
\end{proof}

\subsection{Tools for Proposition \ref{prop:tight-bayes-multi}}
\begin{lem}
\label{lem:lem-for-proposition}
For all $\mathbf{x}\in\mathrm{X}_{\sss\mathcal{U}}$, for all $(i,j)\in \{1,\ldots,K\}^2,$ the following inequality holds:
\begin{multline}
\label{eq:prop-multi:1.1}
R_\mathcal{U}(B_Q,i,j) \geq \frac{1}{u_i}\sum_{\mathbf{x}\in\mathrm{X}_{\sss\mathcal{U}}}P(Y=i|X=\mathbf{x})\I{B_Q(\mathbf{x})=j}\I{v_Q(\mathbf{x},j)<\gamma^*} \\+ \frac{1}{\gamma^*}\floor*{\floor{K_{i,j}-M_{i,j}^<(\gamma^*)}_+ - r_{i,j}}_+ + r_{i,j},
\end{multline}
where $\gamma^* := \sup\{\gamma\in\Gamma_j|\sum_{\mathbf{x}\in\mathrm{X}_{\sss\mathcal{U}}}P(Y=i|X=\mathbf{x})\I{B_Q(\mathbf{x})=j}\I{v_Q(\mathbf{x},j)=\gamma}/u_i> \tau\}$.
\end{lem}
\begin{proof}
Denote $\gamma^*=\gamma_j^{(p)}$. 
According to Lemma \ref{lem:connection-Gibbs-Bayes-multi}, 
$K_{i,j} = \sum_{n=1}^{N_j} b^{(n)}_{i,j}\gamma^{(n)}_j$
, where $b_{i,j}^{(n)} := \frac{1}{u_i}\sum_{\mathbf{x}\in\mathrm{X}_{\sss\mathcal{U}}}P(Y=i|X=\mathbf{x})\I{B_Q(\mathbf{x})=j}\I{v_Q(\mathbf{x},j)=\gamma^{(n)}_j}$.
We can express $b_{i,j}^{(p)}$ in the following way:
\[
b_{i,j}^{(p)} = \frac{K_{i,j} - \sum_{n=1}^{p-1} b^{(n)}_{i,j}\gamma^{(n)}_j - \sum_{n=p+1}^{N_j} b^{(n)}_{i,j}\gamma^{(n)}_j}{\gamma_j^{(p)}} = \frac{K_{i,j} - \sum_{n=1}^{p-1} b^{(n)}_{i,j}\gamma^{(n)}_j - r_{i,j}}{\gamma_j^{(p)}}.
\]
Remind $B^{(n)}_{i,j} = \frac{1}{u_i}\sum_{\mathbf{x}\in\mathrm{X}_{\sss\mathcal{U}}}P(Y=i|X=\mathbf{x})\I{v_Q(\mathbf{x},j)=\gamma^{(n)}_j}$. From this we derive the following:
$$-\sum_{n=1}^{p-1} b^{(n)}_{i,j}\gamma^{(n)}_j \geq -\sum_{n=1}^{p-1} B^{(n)}_{i,j}\gamma^{(n)}_j = - M_{i,j}^<(\gamma_j^{(p)})= - M_{i,j}^<(\gamma^*).$$
Taking into account this as well as $b_{i,j}^{(p)}\geq 0$, we deduce a lower bound for $b_{i,j}^{(p)}$:
\begin{equation}
\label{eq:prop-multi:2}
b_{i,j}^{(p)}\geq\frac{1}{\gamma^*}\floor{K_{i,j}-M_{i,j}^<(\gamma^*) - r_{i,j}}_+ = \frac{1}{\gamma^*}\floor*{\floor{K_{i,j}-M_{i,j}^<(\gamma^*)}_+ - r_{i,j}}_+.
\end{equation}
Also, taking into account Lemma \ref{lem:connection-Gibbs-Bayes-multi}, one can notice that:
\begin{align}
\label{eq:prop-multi:3}
R_\mathcal{U}(B_Q,i,j) &= \sum_{n=1}^{N_j} b_{i,j}^{(n)}
=\sum_{n=1}^{p-1}b_{i,j}^{(n)} + b_{i,j}^{(p)} + \sum_{n=p+1}^{N_j}b_{i,j}^{(n)} \nonumber \\
&\geq \frac{1}{u_i}\sum_{\mathbf{x}\in\mathrm{X}_{\sss\mathcal{U}}}P(Y=i|X=\mathbf{x})\I{B_Q(\mathbf{x})=j}\I{v_Q(\mathbf{x},j)<\gamma^*} + b_{i,j}^{(p)} + r_{i,j},
\end{align}
since $\sum_{n=p+1}^{N_j}b_{i,j}^{(n)}\geq \sum_{n=p+1}^{N_j}b_{i,j}^{(n)}\gamma_j^{(n)}$.
Combining Eq. \eqref{eq:prop-multi:2} and Eq. \eqref{eq:prop-multi:3} we infer Eq. \eqref{eq:prop-multi:1.1}:
\begin{multline*}
R_\mathcal{U}(B_Q,i,j) \geq \frac{1}{u_i}\sum_{\mathbf{x}\in\mathrm{X}_{\sss\mathcal{U}}}P(Y=i|X=\mathbf{x})\I{B_Q(\mathbf{x})=j}\I{v_Q(\mathbf{x},j)<\gamma^*} \\+ \frac{1}{\gamma^*}\floor*{\floor{K_{i,j}-M_{i,j}^<(\gamma^*)}_+ - r_{i,j}}_+ + r_{i,j}.
\end{multline*}
\end{proof}

\section{Tools for Section \ref{sec:c-bound}}
\label{sec:appendix-cbound}
\subsection{Tools for Theorem \ref{thm:prob-cbound}}
\begin{lem}[Cantelli-Chebyshev inequality][Ex 2.3 in \cite{MassartBook}]
\label{lem:cantelli-chebyshev}
    Let $Z$ be a random variable with the mean $\mu$ and the variance $\sigma^2$. Then, for every $a>0$, we have:
    \[
    P(Z\leq \mu - a) \leq \frac{\sigma^2}{\sigma^2 + a^2}.
    \]
\end{lem}
\subsection{Tools for Theorem \ref{thm:pac-bayesian-cbound}}

\subsubsection{Bounds for the Mislabeling Matrix' Entries}
We remind that the imperfection is summarized through the mislabeling matrix $\mathbf{P} = (p_{i,j})_{1\leq i,j \leq K}$ with 
\begin{align*}
    p_{i, j} := P(\hat Y=i|Y=j)  \quad\text{ for all } (i,j)\in \{1,\dots,K\}^2
\end{align*}
such that $\sum_{i=1}^K p_{i,j} = 1$. Also, recall that $\delta(\mbf{x}) := p_{B_Q(\mbf{x}), B_Q(\mbf{x})} -  \max_{j\in\mathcal{Y}\setminus\{B_Q(\mbf{x})\}} p_{B_Q(\mbf{x}), j}$ and $\alpha(\mbf{x})=p_{B_Q(\mbf{x}), B_Q(\mbf{x})}$.

\begin{prop} 
\label{prop:pac-bound-mislab-mat}
Let $\mbf{P}$ be the mislabeling matrix, and assume that $p_{i,i}> p_{i,j},\ \forall{i,j}\in\{1,\dots,K\}^2$. For any $\epsilon \in (0,1]$, with probability $1-\epsilon$ over the choice of the $l$ sample, for all $(i,j)\in \{1,\ldots,K\}^2$, for all $\mbf{x}\in\mathcal{X}$,
\begin{align}
    &\hat{p}_{j,c} - r(l_c) \leq p_{j,c} \leq \hat{p}_{j,c} +  r(l_c), \label{eq:pac-mislab-entry}\\
    &\alpha(\mbf{x}) \leq \hat{\alpha}(\mbf{x}) +  r(l_{c_\mbf{x}}),  \label{eq:pac-alpha}\\
    &\frac{1}{\delta(\mbf{x})} \leq \frac{1}{\hat{\delta}(\mbf{x}) - r(l_{c_\mbf{x}}) - r(l_{j_\mbf{x}})},\ \text{ if } \hat{\delta}(\mbf{x}) \geq r(l_{c_\mbf{x}}) + r(l_{j_\mbf{x}}), \label{eq:pac-delta}
\end{align}
where 
\begin{itemize}
    \item $r(l_k) = \sqrt{\frac{1}{2l_k}\ln\frac{2\sqrt{l_k}}{\epsilon}}$,
    \item $l_k =  \sum_{i = 1}^{l}\I{y_i=k}/l$  is the proportion of the labeled training examples from the true class $k$,
    \item $c_\mbf{x}:=B_Q(\mbf{x})$, $j_\mbf{x}:=\argmin_{j\in\mathcal{Y}\setminus\{c_\mbf{x})\}}l_j$,
    \item $\hat{p}_{j,c}$, $\hat{\alpha}(\mbf{x})$ and $\hat{\delta}(\mbf{x})$ are empirical estimates respectively of $p_{j,c}$, $\alpha(\mbf{x})$ and $\delta(\mbf{x})$ based on the available $l$ sample.
\end{itemize}
\end{prop}
\begin{proof}
Let $S_{j}$ denote the subset of the available examples for which the true class is $j$.
Consider the non-negative random variable $\exp\left\{2 l_j(\hat{p}_{i,j}-p_{i,j} )^2\right\}$. 
From the Markov inequality we obtain that the following holds with probability at least $1-\epsilon$ over $S_j\sim P(\mbf{X}|Y=j)^{l_j}$:
\begin{align}
\label{eq:th-b7-1}
\exp\left\{2 l_j(\hat{p}_{i,j}-p_{i,j} )^2\right\} \leq \frac{1}{\delta}\E_{S_j} \exp\left\{2 l_j(\hat{p}_{i,j}-p_{i,j} )^2 \right\}.
\end{align}
By successively applying Lemma \ref{lem:pinsker} and Lemma \ref{prop:Maurer}, we deduce that
\begin{align}
 \E_{S_j} \exp\left\{2 l_j(\hat{p}_{i,j}-p_{i,j} )^2 \right\} &\leq \E_{S_j} \exp\left\{ l_j\cdot kl(\hat{p}_{i,j}||p_{i,j} ) \right\} 
 \leq 2\sqrt{l_j}. \label{eq:th-b7-2}
\end{align}
 
Combining Eq. \eqref{eq:th-b7-1} and Eq. \eqref{eq:th-b7-2}, we infer $2 l_j(\hat{p}_{i,j}-p_{i,j} )^2\leq \ln\left(2\sqrt{l_j}/\delta\right)$. Eq. \eqref{eq:pac-mislab-entry} is directly obtained from the last inequality, and hence, we derive also Eq. \eqref{eq:pac-alpha}. To prove Eq. \eqref{eq:pac-delta}, let us define 
$$k_{\mbf{x}} := \argmax_{k\in\mathcal{Y}\setminus\{B_Q(\mbf{x})\}} p_{c_\mbf{x}, k}, \qquad \hat{k}_{\mbf{x}} := \argmax_{k\in\mathcal{Y}\setminus\{B_Q(\mbf{x})\}} \hat{p}_{c_\mbf{x}, k}.$$
Then, we write:
\begin{align*}
    \frac{1}{\delta(\mbf{x})} &= \frac{1}{p_{c_\mbf{x}, c_\mbf{x}}-p_{c_\mbf{x}, k_\mbf{x}}}\leq \frac{1}{p_{c_\mbf{x}, c_\mbf{x}}-p_{c_\mbf{x}, k_\mbf{x}} - r(l_{c_\mbf{x}}) - r(l_{k_\mbf{x}})}\\
    &\leq \frac{1}{p_{c_\mbf{x}, c_\mbf{x}}-p_{c_\mbf{x}, \hat{k}_\mbf{x}} - r(l_{c_\mbf{x}}) - r(l_{j_\mbf{x}})} = 
    \frac{1}{\hat{\delta}(\mbf{x}) - r(l_{c_\mbf{x}}) - r(l_{j_\mbf{x}})}.
\end{align*}
These transitions hold only when the denominator is positive, which is ensured if $\hat{\delta}(\mbf{x}) \geq r(l_{c_\mbf{x}}) + r(l_{j_\mbf{x}})$.
\end{proof}

\begin{lem}[Pinsker’s Inequality for Bernoulli random variables, Theorem 4.19 in  \cite{MassartBook}]
\label{lem:pinsker}
For all $p_1,p_2\in[0,1]^2$,
\begin{align*}
    &2(p_2\!-\!p_1)^2 \leq kl(p_2||p_1)\\
    &kl(p_2||p_1)\!:=\! p_2\ln\frac{p_2}{p_1}+(1\!-\!p_2)\ln\frac{1\!-\!p_2}{1\!-\!p_1} = \kld{P_2}{P_1},
\end{align*}
where $P_2$ and $P_1$ are Bernoulli distributions with parameters $p_2$ and $p_1$ respectively.
\end{lem}

\begin{lem}[Theorem 1 in \cite{Maurer:2004} and Lemma 19 in \cite{Germain:2015}] \label{prop:Maurer}
Let $\mathbf{X}=(X_1,\dots, X_n)$ be a random vector, whose components $X_i$ are i.i.d. with values $\in[0,1]$ and expectation $\mu$. Let $\mathbf{X'}=(X_1',\dots, X_n')$ denotes a random vector, where each $X_i'$ is the unique Bernoulli random variable of the corresponding $X_i$: $P(X_i'=1)=\E X_i'=\E X_i=\mu,\ \forall i\in\{1,\dots,n\}$. Then,
\begin{align*}
    \E\left[e^{n\kld{\bar{X}}{\mu}}\right]\leq \E\left[e^{n\kld{\bar{X}'}{\mu}}\right] \leq 2\sqrt{n},
\end{align*}
where $\bar{X} = \frac{1}{n}\sum_{i=1}^n X_i$ and $\bar{X}' = \frac{1}{n}\sum_{i=1}^n X_i'$.
\end{lem}

\subsubsection{Lower Bound of the First Moment of the Margin}
\begin{prop} \label{prop:pac-bayes-bound-first-moment} 
Let $\hat{M}$ be a random variable such that $[\hat{M}|\mbf{X}=\mbf{x}]$ is a discrete random variable that is equal to the margin $M_Q(\mbf{x}, j)$ with probability $P(\hat{Y}\!=\!j|\mbf{X}\!=\!\mbf{x})$, $j=\{1,\dots,K\}$.
Let $\mu^{\hat{M}}_1$  be defined as in Theorem \ref{thm:w-cbound}. Given the conditions of Proposition \ref{prop:pac-bound-mislab-mat}, for any set of classifiers $\mathcal{H}$, for any prior distribution $P$ on $\mathcal{H}$ and any $\epsilon \in (0,1]$, with a probability at least $1-\epsilon$ over the choice of the $n$ sample, for every posterior distribution $Q$ over $\mathcal{H}$
\begin{align*}
    \mu^{\hat{M}, \mbf{P}}_1 \geq \bar\mu^{S}_1 - B_1 \sqrt{\frac{2}{n}\left[\kld{Q}{P} + \ln\frac{2\sqrt{n}}{\delta}\right]},
\end{align*}
where 
\begin{itemize}
    \item $\bar\mu^{S}_1=\frac{1}{n}\sum_{i=1}^n(1/\tilde{\delta}(\mbf{x}))\sum_{c=1}^K M_Q(\mbf{x}, c) P(Y\!=\!c|\mbf{X}\!=\!\mbf{x})$ is the empirical weighted margin mean based on the available $n$-sample $S$,
    \item $\tilde{\delta}(\mbf{x}):=\hat{\delta}(\mbf{x}) - r(l_{c_\mbf{x}}) - r(l_{j_\mbf{x}})$,
    \item $B_1 := \max_{x\in\mathcal{X}}|(1/\tilde{\delta}(\mbf{x}))\sum_{c=1}^K M_Q(\mbf{x}, c) P(Y\!=\!c|\mbf{X}\!=\!\mbf{x})|$,
    \item $KL$ denotes the Kullback–Leibler divergence.
\end{itemize}
\end{prop}
\begin{proof}[Proof]


Further, we denote the available sample with imperfect labels by $S$. 
Let $\mu^{\hat{M}, \mbf{P}, h}_1$ and $\bar\mu^{S, h}_1$ be the random variables such that $\mu^{\hat{M}, \mbf{P}}_1 = \E_{h\sim Q} \mu^{\hat{M}, \mbf{P}, h}_1$ and $\bar\mu^{S}_1 = \E_{h\sim Q} \bar\mu^{S, h}_1$.

We apply the Markov inequality to $\E_{h\sim P}\exp\left\{\frac{n}{2 B_1^2}(\bar\mu^{S, h}_1-\mu^{\hat{M},\mbf{P}, h}_1)^2\right\}$, which is a non-negative random variable, and obtain that with  probability at least $1-\epsilon$ over $S\sim P(\mbf{X}, \hat{Y})^n$:
\begin{align}
    \E_{h\sim P}\exp\left\{\frac{n}{2 B_1^2}(\bar\mu^{S, h}_1-\mu^{\hat{M},\mbf{P}, h}_1)^2\right\}
    \leq \frac{1}{\epsilon} \E_{S} \E_{h\sim P}\exp\left\{\frac{n}{2 B_1^2}(\bar\mu^{S, h}_1\!-\!\mu^{\hat{M},\mbf{P}, h}_1)^2\right\}. \label{eq:markov}
\end{align}
Since the prior distribution $P$ over $\mathcal{H}$ is independent on $S$, we can swap 
$\E_{S}$ and $\E_{h\sim P}$. One can notice that 
$$\frac{1}{2 B_1^2}(\bar\mu^{S, h}_1-\mu^{\hat{M},\mbf{P}, h}_1)^2 = 2\left[\frac{1}{2}(1\!-\!\frac{\bar\mu^{S, h}_1}{B_1})\!-\!\frac{1}{2}(1\!-\!\frac{\mu^{\hat{M},\mbf{P}, h}_1}{B_1})\right]^2,$$
which is the squared of the difference of two random variables that are both between 0 and 1. Then, we successively apply Lemma \ref{lem:pinsker} and Lemma \ref{prop:Maurer} deriving that: 
\begin{align*}
    &\E_{h\sim P} \E_{S} \exp\left\{2n\left[\frac{1}{2}\left(1\!-\!\frac{\bar\mu^{S, h}_1}{B_1}\right)\!-\!\frac{1}{2}\left(1\!-\!\frac{\mu^{\hat{M},\mbf{P}, h}_1}{B_1}\right)\right]^2\right\} \\
    &\leq \E_{h\sim P} \E_{S} \exp\left\{n\cdot kl\left(\frac{1}{2}(1\!-\!\frac{\bar\mu^{S, h}_1}{B_1})\right|\left|\frac{1}{2}(1\!-\!\frac{\mu^{\hat{M},\mbf{P}, h}_1}{B_1})\right) \right\} 
    \leq \E_{h\sim P} 2\sqrt{n} = 2\sqrt{n}.
\end{align*}

We apply this result for Eq. \eqref{eq:markov}, and by taking the natural logarithm from the both sides we obtain that:
\begin{align}
\label{eq:bounded-by-2sqrtn}
    \ln\left(\E_{h\sim P}\exp\left\{\frac{n}{2 B_1^2}(\bar\mu^{S, h}_1\!-\!\mu^{\hat{M},\mbf{P}, h}_1)^2\right\}\right) \leq  \ln\left(\frac{2\sqrt{n}}{\epsilon}\right).
\end{align}

Using the change of measure (Lemma \ref{lem:seldin-lem}) and the Jensen's inequalities, we derive that:
\begin{align*}
\ln\left(\E_{h\sim P}\exp\left\{\frac{n}{2 B_1^2}(\bar\mu^{S, h}_1-\mu^{\hat{M},\mbf{P}, h}_1)^2\right\}\right) 
&\geq \E_{h\sim Q} \frac{n}{2 B_1^2}(\bar\mu^{S, h}_1-\mu^{\hat{M},\mbf{P}, h}_1)^2 - \kld{Q}{P}\\
&\geq \frac{n}{2 B_1^2}(\E_{h\sim Q}\bar\mu^{S, h}_1- \E_{h\sim Q}\mu^{\hat{M},\mbf{P}, h}_1)^2 - \kld{Q}{P}.
\end{align*}

Combining with Eq. \eqref{eq:bounded-by-2sqrtn}, we derive:
\begin{align}
\label{eq:almost-final}
\frac{n}{2 B_1^2}(\bar\mu^{S}_1- \mu^{\hat{M},\mbf{P}}_1)^2 \leq \ln\left(\frac{2\sqrt{n}}{\epsilon}\right) + \kld{Q}{P}.
\end{align}
The final inequality is directly inferred from Eq. \eqref{eq:almost-final}.

\begin{lem}[Change of Measure Inequality \cite{Donsker:1975}]
\label{lem:seldin-lem}
For any measurable function $
\phi$ defined on the hypothesis space $\mathcal{H}$ and all distributions $P, Q$ on $\mathcal{H}$, the following inequality holds:
\[
\E_{h\sim Q}\phi(h) \leq \kld{Q}{P} + \ln\E_{h\sim P}e^{\phi(h)}.
\]
\end{lem}
\end{proof}

\subsubsection{Other Required Bounds}

\begin{prop}
\label{prop:pac-bayes-bound-second-moment}
Let $\hat{M}$ be a random variable such that $[\hat{M}|\mbf{X}=\mbf{x}]$ is a discrete random variable that is equal to the margin $M_Q(\mbf{x}, j)$ with probability $P(\hat{Y}\!=\!j|\mbf{X}\!=\!\mbf{x})$, $j=\{1,\dots,K\}$.
Let $\mu^{\hat{M}, \mbf{P}}_2$  be defined as in Theorem \ref{thm:w-cbound}. Given the conditions of Proposition \ref{prop:pac-bound-mislab-mat}, for any set of classifiers $\mathcal{H}$, for any prior distribution $P$ on $\mathcal{H}$ and any $\epsilon \in (0,1]$, with a probability at least $1-\epsilon$ over the choice of the $n$ sample, for every posterior distribution $Q$ over $\mathcal{H}$
\begin{align*}
    \mu^{\hat{M}, \mbf{P}}_2 \leq \bar\mu^{S}_2 + B_2 \sqrt{\frac{2}{n}\left[2\kld{Q}{P} + \ln\frac{2\sqrt{n}}{\epsilon}\right]},
\end{align*}
where 
\begin{itemize}
    \item $\bar\mu^{S}_2=\frac{1}{n}(1/\tilde{\delta}(\mbf{x}))\sum_{i=1}^n\sum_{c=1}^K (M_Q(\mbf{x}_i, c))^2 P(Y\!=\!c|\mbf{X}\!=\!\mbf{x}_i)$ is the empirical weighted 2nd margin moment based on the available $n$-sample $S$,
    \item $\tilde{\delta}(\mbf{x}):=\hat{\delta}(\mbf{x}) - r(l_{c_\mbf{x}}) - r(l_{j_\mbf{x}})$,
    \item $B_2 := \max_{x\in\mathcal{X}}|(1/\tilde{\delta}(\mbf{x}))\sum_{c=1}^K (M_Q(\mbf{x}, c))^2 P(Y\!=\!c|\mbf{X}\!=\!\mbf{x})|$,
    \item $KL$ denotes the Kullback–Leibler divergence.
\end{itemize}
\end{prop}

\begin{proof}
The proof is similar to the one given for Proposition \ref{prop:pac-bayes-bound-first-moment}, but relies on the extension of the change of measure inequality (Lemma \ref{lem:laviolette-2017}).
\begin{lem}[Change of Measure Inequality for Pairs of Voters (Lemma 1 in \cite{Laviolette:2017})]
\label{lem:laviolette-2017}
 For any set of voters $\mathcal{H}$, for any distributions $P, Q$ on $\mathcal{H}$, and for any measurable function $\phi:\ \mathcal{H}\times\mathcal{H}\to\R$, the following inequality holds:
\[
\E_{(h,h')\sim Q^2}\phi(h, h') \leq 2\kld{Q}{P} + \ln\E_{(h, h')\sim P^2}e^{\phi(h, h')}.
\]
\end{lem}
\end{proof}

\begin{prop}
\label{prop:pac-bound-psi}
Given the conditions of Proposition \ref{prop:pac-bound-mislab-mat}, for  any $\epsilon \in (0,1]$, with a probability at least $1-\epsilon$ over the choice of the $n$ sample,
\begin{align*}
    \psi_{\mbf{P}} \leq \frac{1}{n}\sum_{i=1}^n \frac{\hat\alpha(\mbf{x}_i)+r(l_{c_\mbf{x}})}{\hat\delta(\mbf{x}_i)-r(l_{c_\mbf{x}})-r(l_{j_\mbf{x}})} + B_3 \sqrt{\frac{2}{n} \ln\frac{2\sqrt{n}}{\epsilon}},
\end{align*}
where $B_3 := \max_{x\in\mathcal{X}}[\hat\alpha(\mbf{x}_i)+r(l_{c_\mbf{x}})]/[\hat\delta(\mbf{x}_i)-r(l_{c_\mbf{x}})-r(l_{j_\mbf{x}})]$.
\end{prop}
\begin{proof}
First, we take into consideration the result of Proposition \ref{prop:pac-bound-mislab-mat} and deduce that $\psi_{\mbf{P}} \leq \E_{\mbf{X}} [(\hat\alpha(\mbf{x}_i)+r(l_{c_\mbf{x}}))/(\hat\delta(\mbf{x}_i)-r(l_{c_\mbf{x}})-r(l_{j_\mbf{x}}))]$. The rest of proof is similar to those are given for Proposition \ref{prop:pac-bound-mislab-mat} and for Proposition \ref{prop:pac-bayes-bound-first-moment}.
\end{proof}

\section{Additional Experiments}
\subsection{Approximation of the Posterior Probabilities for Self-learning}
\label{sec:posterior-estimation}
In this section, we analyze the behavior of \texttt{MSLA} depending on how the transductive bound given by Eq.~\eqref{eq:tr-bound-joint-bayes-multi} is evaluated. Since the posterior probabilities for unlabeled data are not known, we have proposed to estimate them as the votes of the base supervised classifier learned using the labeled data only (Sup. Estimation). This approach has been used in Section \ref{sec:num-exper} for running \texttt{MSLA}. We compare it with another strategy that is to assign $P(Y=i|\mbf{X}=\mbf{x})=1/K,\ \forall \mbf{x}\in\mathrm{X}_{\sss\mathcal{U}},\ \forall i\in\{1,\dots,K\}$. In this case, we consider the worst case when every class is equally probable for each example (Unif. Estimation). Finally, we provide the performance of \texttt{MSLA} when the labels of unlabeled data are given, which means that the transductive bound is truly estimated (Oracle). Table \ref{tab:tr-bound-prob-estim} illustrates the performance results. As we can see, the supervised approximation generally outperforms the uniform one (significantly on \texttt{MNIST}). This might be explained by the fact that the supervised votes may give some additional information on the most probable labels for each example. In addition, we have observed that on the last iterations the votes of \texttt{MSLA} tend to be biased, so such posteriors can play a role of regularization. The performance results of the oracle show that better estimation of the posteriors can give an improvement, though not significantly on most of data sets. Note that the performance of the oracle is not perfect, because the true labels are used only for the bound estimation, and the votes are used for pseudo-labeling.



\label{sec:exp-prob-estim}
\begin{table}[h]
    \centering
      {\scalebox{0.95}{
    \begin{tabular}{c|ccc}
    \toprule
    \multirow{2}{*}{Data set} & \multicolumn{3}{c}{\texttt{MSLA}} \\
    \cline{2-4}
    & Unif. Estimation & Sup. Estimation & Oracle \\
    \midrule
    \texttt{Vowel}  &  .586 $\pm$ .029 &  .586 $\pm$ .026 &  .599 $\pm$ .028 \\
    \midrule
    \texttt{Protein}  &  .773 $\pm$ .034 &  .781 $\pm$ .034 &  .805 $\pm$ .036 \\
    \midrule
    \texttt{DNA}  &   .697 $\pm$ .079 &  .702 $\pm$ .082 &   .721 $\pm$ .09 \\
    \midrule
    \texttt{Page Blocks}  &  .965 $\pm$ .002 &  .966 $\pm$ .002 &  .966 $\pm$ .002 \\
    \midrule
    \texttt{Isolet}  &  .869 $\pm$ .015 &  .875 $\pm$ .014 &  .885 $\pm$ .012 \\
    \midrule
    \texttt{HAR}  &  .852 $\pm$ .025 &  .854 $\pm$ .026 &  .856 $\pm$ .022 \\
    \midrule
    \texttt{Pendigits}  &   .873 $\pm$ .024 &  .884 $\pm$ .022 &  .892 $\pm$ .016 \\
    \midrule
    \texttt{Letter}  &  .716 $\pm$ .013 &  .717 $\pm$ .013 &  .723 $\pm$ .012 \\
    \midrule
    \texttt{Fashion}  &  .722 $\pm$ .022 &  .723 $\pm$ .023 &  .728 $\pm$ .024 \\
    \midrule
    \texttt{MNIST} &  .834 $\pm$ .016 &  .857 $\pm$ .013 &   .87 $\pm$ .012 \\
    \midrule
    \texttt{SensIT} &  .722 $\pm$ .021 &  .722 $\pm$ .021 &  .722 $\pm$ .021 \\
    \bottomrule
    \end{tabular}}}
    \caption{The performance comparison of \texttt{MSLA} depending on how the posterior probabilities are estimated in the evaluation of the transductive bound (Eq. \eqref{eq:tr-bound-joint-bayes-multi}).}
    \label{tab:tr-bound-prob-estim}
\end{table}

\subsection{Time}
\label{sec:run-time}

In this section, we present the run-time of all the algorithms empirically compared in Section \ref{sec:msla-exp}. The results are depicted in Table\ref{tab:computationTime}. In general, the obtained run-time is coherent with the complexity analysis presented in Section \ref{sec:msla-exp}. \texttt{LS} and \texttt{QN-S3VM} have a very large run-time when they converge slowly, and they are generally slower than the other algorithms. \texttt{Semi-LDA} is fast on the considered data sets, though it may slow down on data of large dimension not considered in this paper.

It can be seen that \texttt{DAS-RF} is slower than the self-learning algorithms, which is due to the fact that the classifier is trained on all labeled and unlabeled examples at each iteration. \texttt{CSLA} is the fastest approach since it re-trains the base classifier only 3 times compared to 10 times for \texttt{FSLA}. From our observation, \texttt{MSLA} needs usually around 3-5 iterations to pseudo-label the whole unlabeled set, but it takes more time than \texttt{CSLA}, since it searches at each iteration the threshold by minimizing the conditional Bayes error. We have implemented the search in a single core, but it can be potentially parallelized. Nevertheless, the \texttt{MSLA} still runs fast. 

\begin{table}[ht!]
\caption{The average run-time of the learning algorithms under consideration on the data sets described in Table \ref{tab:data set-description}. $s$~stands for seconds, $m$ for minutes and $h$ for hours.}
\label{tab:computationTime}
\hfill \break
\centering
\scalebox{0.95}
{
\begin{tabular}{l| cccccccc}
    \toprule
    Data set &
    \texttt{RF} & \texttt{LS} & \texttt{QN-S3VM} & \texttt{Semi-LDA} & \texttt{DAS-RF} & \texttt{FSLA$_{\theta=0.7}$} & 
    \texttt{CSLA$_{\Delta=1/3}$} & \texttt{MSLA}\\
    \midrule
    \texttt{Vowel}  &  1\,s & 6\,s & 2\,s & 3\,s & 7\,s & 11\,s & 2\,s & 5\,s\\
    \midrule
    \texttt{Protein}  & 1\,s &  22\,s & 4\,m & 5\,s & 6\,s & 10\,s & 2\,s & 4\,s\\
    \midrule
    \texttt{DNA}  &   1\,s & 1\,m & 26\,s & 1\,s & 9\,s & 7\,s & 3\,s & 4\,s\\
    \midrule
    \texttt{PageBlocks}  & 1\,s &  2\,m & 2\,m & 14\,s & 9\,s & 12\,s & 3\,s & 6\,s\\
    \midrule
    \texttt{Isolet}  &  1\,s & 1\,m & 1\,h & 10\,s & 38\,s & 16\,s & 5\,s & 28\,s\\
    \midrule
    \texttt{HAR}  &  1\,s & 18\,m & 32\,m & 3\,s & 42\,s & 23\,s & 6\,s & 13\,s\\
    \midrule
    \texttt{Pendigits}  &  1\,s & 30\,m & 10\,m & 37\,s & 13\,s & 13\,s & 3\,s & 14\,s\\
    \midrule
    \texttt{Letter}  &  1\,s & 3\,h & 40\,m & 1\,m & 20\,s & 16\,s & 5\,s & 1\,m\\
    \midrule
    \texttt{Fashion}  & 1\,s & $>$4\,h & $>$4\,h &  1\,m & 2\,m & 1\,m & 29\,s & 1\,m\\
    \midrule
    \texttt{MNIST} &  1\,s & $>$4\,h & $>$4\,h & 1\,m & 2\,m & 1\,m & 29\,s & 1\,m\\
    \midrule
    \texttt{SensIT} & 1\,s & $>$4\,h & $>$4\,h & 2\,m & 3\,m & 2\,m & 30\,s & 1\,m\\
    \bottomrule
\end{tabular}}
\end{table}

\subsection{Relaxation of CBIL}
\label{sec:relax_bound}

The proposed \eqref{eq:w-cbound} is based on Eq. \eqref{eq:one-x-mislabel-ineq}, which holds only when $\delta(\mbf{x})\geq 0$. As it was discussed in Section \ref{sec:mislab-error-model}, Eq. \eqref{eq:one-x-mislabel-ineq} can be relaxed by adding some $\lambda>0$ leading to Eq. \eqref{eq:one-x-mislabel-ineq-with-lam}. In practice, it not only can make the bound computable, but also make it smoother, since arbitrarily small values of $\delta(\mbf{x})$ implies arbitrarily large values of $\hat{r}(\mbf{x})/\delta(\mbf{x})$. The latter should be avoided if \eqref{eq:w-cbound} is used as some optimization or selection criterion.

In this section, we study the impact of $\lambda$ on the bound's value on different data sets. In Figure \ref{fig:cbil-lams}, we display the results of all 20 experimental trials for \texttt{HAR}, \texttt{Isolet}, \texttt{Letter}, \texttt{MNIST} and \texttt{Fashion} when $\lambda\in[0.1, 0.2, \dots, 1]$. One can observe that when the bound is not penalized much (i.e., $\delta(\mbf{x})$ is far from 0), then the increase of $\lambda$ makes the bound looser, so $\lambda=0.1$ is the tightest choice. Exactly the opposite situation is observed when $\delta(\mbf{x})$ is small (trials 4 and 14 for \texttt{Letter}, most of trials for \texttt{Fashion}): higher values of $\lambda$ diminish the influence of hyperbolic weights $1/\delta(\mbf{x})$, so $\lambda=1$ leads to the tightest bound.

\begin{figure}[ht!]
    \centering
    \includegraphics[width=\textwidth]{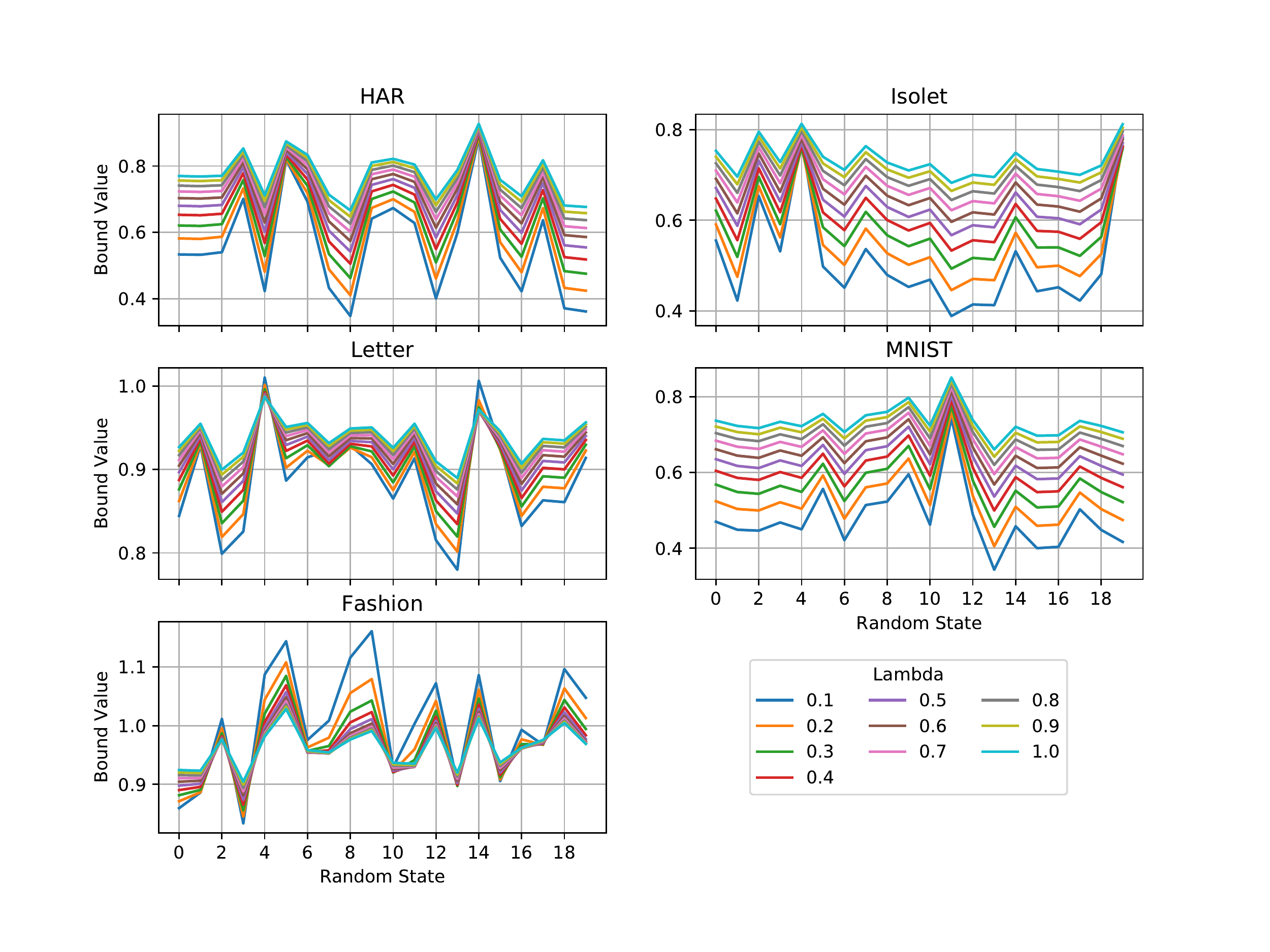}
        \caption{The value of \eqref{eq:w-cbound} with different $\lambda$ over 20 different labeled/unlabeled splits of 5 data sets.}
    \label{fig:cbil-lams}
\end{figure}

We also note that small $\delta(\mbf{x})$ not only makes the bound looser, but also leads to poor correlation with the true error. It can particularly be seen in Figure \ref{fig:fashion-cbound}, where we repeated the experiment done in Section \ref{sec:cbil-exp} for the \texttt{Fashion} data set when $\lambda=0$ and $\lambda=0.1$. It is clearly seen that with $\lambda=0.1$ the curve's shape becomes much more similar to the oracle C-bound. Eventually, in average, $\lambda$ can make the bound looser but better correlated with the true error, where the latter is more important for practical applications.



\begin{figure}[ht!]
    \centering
    \includegraphics[width=\textwidth]{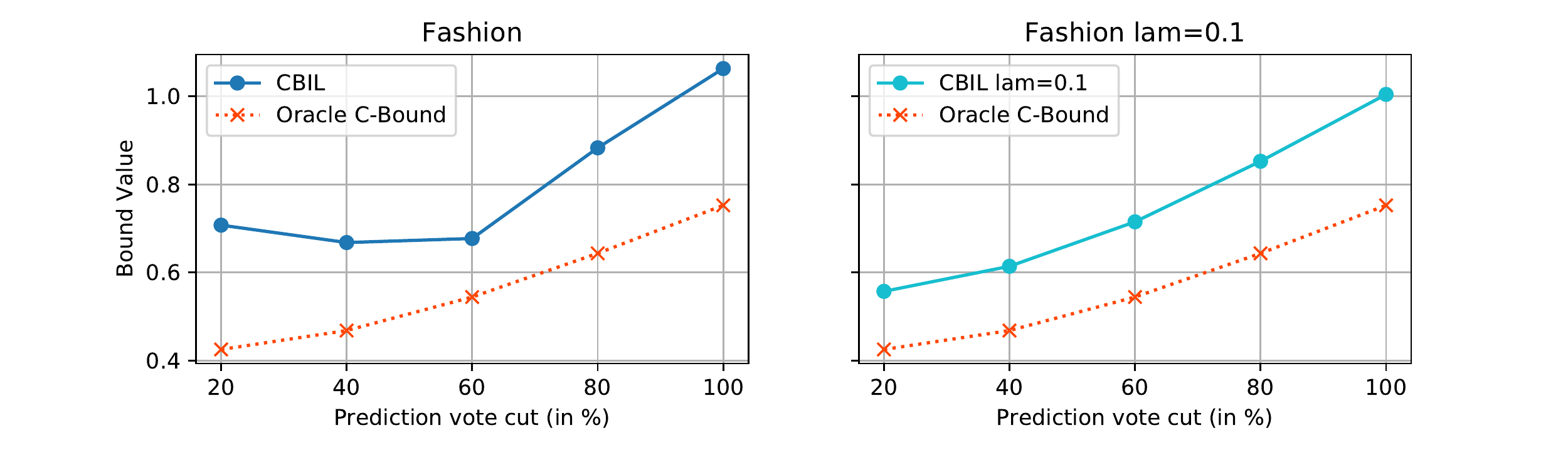}
        \caption{\eqref{eq:w-cbound} and Oracle C-Bound when varying the number of unlabeled examples used for evaluation on Fashion data set. We keep the most confident one (with respect to prediction vote) from $20\%$ to $100\%$.}
    \label{fig:fashion-cbound}
\end{figure}

\end{document}